\documentclass[]{fairmeta}
\PassOptionsToPackage{capitalize,noabbrev,nameinlink}{cleveref}

\PassOptionsToPackage{numbers, compress}{natbib}

\usepackage[utf8]{inputenc} 
\usepackage[T1]{fontenc}    
\usepackage{hyperref}       
\usepackage{url}            
\usepackage{booktabs}       
\usepackage{amsfonts}       
\usepackage{nicefrac}       
\usepackage{microtype}      
\usepackage{xcolor}         
\usepackage{pifont}
\usepackage{makecell} 
\usepackage{tcolorbox}

\usepackage{amsfonts}       
\usepackage{colortbl} 
\usepackage{nicefrac}       
\usepackage{microtype}      
\usepackage{xcolor}         
\usepackage{xspace}
\usepackage{placeins}  
\usepackage{titletoc}
\usepackage{amsmath}
\usepackage{booktabs}
\usepackage{multirow}
\usepackage{mathtools}
\usepackage{array}
\usepackage{amsthm}
\usepackage{verbatim} 
\usepackage{enumerate}
\usepackage{bbm}
\usepackage{commath}
\usepackage{wrapfig}
\usepackage{subcaption} 
\usepackage{amsbsy}
\usepackage{float}
\usepackage{algorithm}
\usepackage{algpseudocode}
\usepackage{listings}
\usepackage{caption}
\usepackage{enumitem}
\usepackage{lipsum}
\usepackage{listings}
\usepackage[colorinlistoftodos]{todonotes}
\usepackage{wrapfig, blindtext}

\newtheorem{lemma}{Lemma}
\newtheorem{corollary}{Corollary}
\theoremstyle{definition}

\theoremstyle{remark}
\newtheorem{remark}{Remark}
\usepackage{thmtools}
\usepackage{thm-restate}

\DeclareMathOperator{\rank}{rank}

\definecolor{codegreen}{rgb}{0,0.6,0}
\definecolor{codegray}{rgb}{0.5,0.5,0.5}
\definecolor{codepurple}{rgb}{0.58,0,0.82}
\definecolor{backcolour}{rgb}{0.95,0.95,0.92}
\definecolor{codepurple}{rgb}{0.58,0,0.82}
\definecolor{papercolor}{HTML}{0668E1}
\definecolor{darkred}{rgb}{0.68,0.05,0.0}
\definecolor{tab_blue}{RGB}{230,245,255} 
\definecolor{tab_purple}{RGB}{245,230,255} 
\definecolor{codeblue}{rgb}{0.25,0.5,0.5}
\definecolor{codekw}{rgb}{0.85, 0.18, 0.50}
\definecolor{mygreen}{RGB}{0,128,0}
\definecolor{emerald}{RGB}{34,139,34}

\newcommand{\algoname}{\textsc{ReasonCache }}
\newcommand{\algofullname}{\textsc{ReasonCache }}

\definecolor{yesgreen}{HTML}{009901}
\definecolor{nored}{HTML}{FF0000}

\definecolor{mmqaLight}{HTML}{D9EAF4}   
\definecolor{ot3Light}{HTML}{F2D6E4}    

\newcolumntype{M}{>{\centering\arraybackslash\columncolor{mmqaLight}}c}
\newcolumntype{O}{>{\centering\arraybackslash\columncolor{ot3Light}}c}

\title{Reason\textsc{CACHE}\xspace: Teaching LLMs To Reason  \\ \vspace{-6mm} Without Weight Updates \\
}

\author[1,2,*]{Sharut Gupta}
\author[2]{Phillip Isola}
\author[2,3]{Stefanie Jegelka}
\author[1]{David Lopez-Paz}
\author[1]{Kartik Ahuja}
\author[1]{Mark Ibrahim}
\author[1]{Mohammad Pezeshki}

\affiliation[1]{FAIR at Meta}
\affiliation[2]{MIT CSAIL}
\affiliation[3]{TU Munich}

\contribution[*]{Work done during an internship at Meta}

\abstract{
Can Large language models (LLMs) learn to reason without any weight update and only through in-context learning (ICL)? ICL is strikingly sample-efficient, often learning from only a handful of demonstrations, but complex reasoning tasks typically demand many training examples to learn from. However, naively scaling ICL by adding more demonstrations breaks down at this scale: attention costs grow quadratically, performance saturates or degrades with longer contexts, and the approach remains a shallow form of learning. Due to these limitations, practitioners predominantly rely on in-weight learning (IWL) to induce reasoning. In this work, we show that by using \textit{Prefix Tuning}, LLMs can learn to reason without overloading the context window and without any weight updates. We introduce \textsc{ReasonCache}, an instantiation of this mechanism that distills demonstrations into a fixed key-value cache. Empirically, across challenging reasoning benchmarks, including GPQA-Diamond, \textsc{ReasonCache} outperforms standard ICL and matches or surpasses IWL approaches. Further, it achieves this all while being more efficient across three key axes: data, inference cost, and trainable parameters. We also theoretically prove that \textsc{ReasonCache} can be strictly more expressive than low-rank weight update since the latter ties expressivity to input rank, whereas \textsc{ReasonCache} bypasses this constraint by directly injecting key-values into the attention mechanism. Together, our findings identify \algoname as a middle path between in-context and in-weight learning, providing a scalable algorithm for learning reasoning skills beyond the context window without modifying parameters.\looseness=-1}

\date{\today}
\correspondence{\email{sharut@mit.edu}, \email{mpezeshki@meta.com}}
\metadata[Blogpost]{\url{https://reasoncache.github.io/}}

\begin{document}
\maketitle

\section{Introduction}
In-context learning (ICL) represents one of the most remarkable capabilities of modern large language models. By placing demonstrations directly in the prompt, ICL elicits complex behaviors (shifting styles, formats, and task competencies) without a single gradient update \citep{brown2020language,min2022rethinking,agarwal2024many,eyuboglu2025cartridges,gupta2023context,petrov2023prompting,wang2025prefix,yin2024deeper}. In many regimes, especially when data is limited, this inference-time conditioning often matches or even exceeds task-specific fine-tuning on a range of tasks such as factual question answering, puzzles, summarization or grade school mathematics~\citep{wei2022chain,zhou2022least,akyurek2024surprising,eyuboglu2025cartridges,agarwal2024many,lampinen2025generalization,yang2024synthetic,si2022prompting,awadalla2022exploring}.

ICL\footnote{Throughout, we use the term \emph{in-context learning} broadly to denote any form of adaptation driven by information provided in the context (examples, instructions, or intermediate reasoning) without updating model parameters.}, however, struggles with demanding reasoning tasks; complex mathematical problem-solving and novel algorithmic skills rarely emerge just from a handful of demonstrations. While populating the prompt with more examples appears to be a natural workaround, standard ICL soon hits scaling limits, for three key reasons. 1) It is computationally burdensome and capacity-limited; the finite context window during training bounds how many demonstrations fit, and attention’s quadratic scaling causes inference latency and memory usage to grow rapidly; 2) It is unreliable at length; the advertised window for long context models often overstates how much context the model can use reliably, and performance often saturates and even degrades as relevant information is pushed farther away \citep{agarwal2024many,zhang2025more,zhang2025memory,liu2024lost}; 3) Most critically, concatenating examples remains a shallow form of adaptation and fails to synthesize the novel reasoning pathways required for complex logic~\citep{guo2025deepseek, mosbach2023few,de2025context,geng2024great}. 
For these reasons, practitioners typically relegate reasoning to parameter updates, typically through supervised fine-tuning or reinforcement learning, while using standard ICL primarily to steer model behavior rather than to acquire new reasoning capabilities~\citep{muennighoff2025s1,guha2025openthoughts,ye2025limo,yuan2025naturalreasoning,yu2023metamath,yue2023mammoth,su2025parametricrag,caccia2025training,kuratov2025cramming}. We revisit this prevailing paradigm and ask whether ICL’s observed limitations are truly fundamental, or instead reflect the particular way in-context supervision is currently represented and scaled. This leads us to the question

\begin{center}
\begin{tcolorbox}[colback=gray!5!white,
                  colframe=black!50,
                  boxrule=0.9pt,
                  arc=2pt,
                  width=1.\linewidth]
\emph{How can in-context learning be scaled into a mechanism for reasoning and what are its implications?}
\end{tcolorbox}
\end{center}

In this work, we demonstrate that prefix tuning (PT) \citep{li2021prefix}, an often-overlooked form of in-context adaptation, provides the ideal interface for scaling ICL and answering this question. We refer to this reasoning instantiation of PT as \textsc{ReasonCache}. Specifically, \algoname learns a small set of prefix key–value vectors at each attention layer while keeping all pretrained weights frozen, so that training compresses the effect of many demonstrations into a compact \emph{cache}.

\begin{figure}
    \centering
    \includegraphics[width=1.\linewidth]{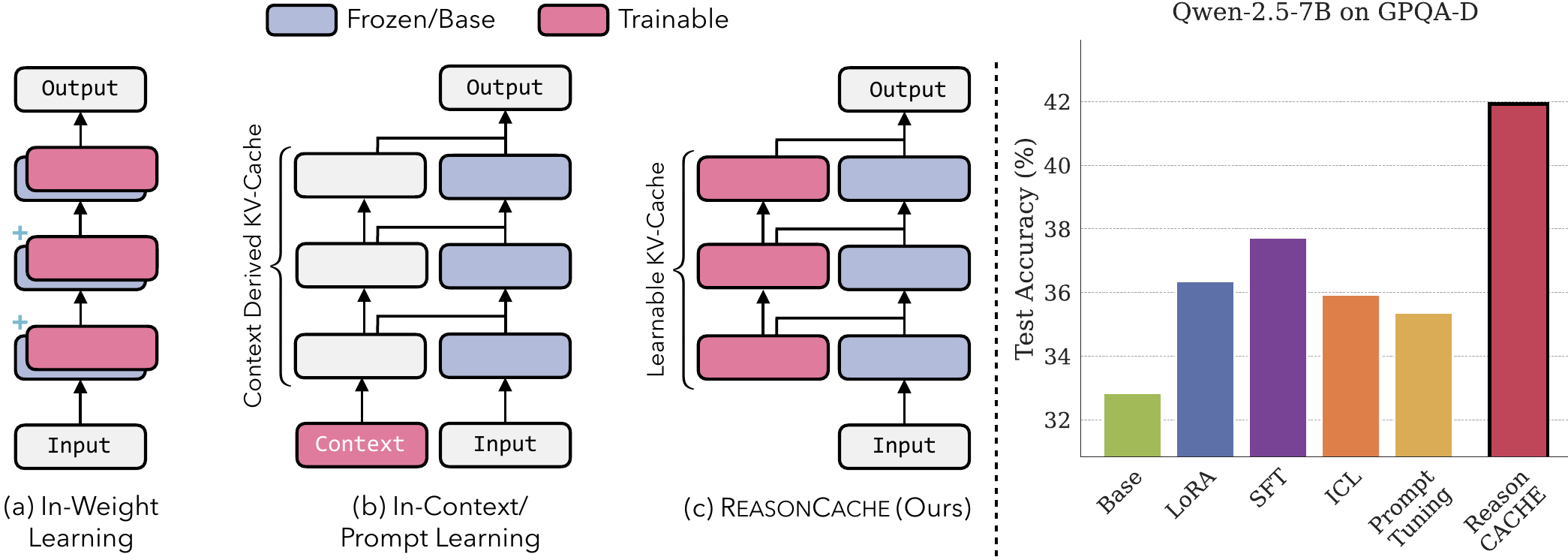}
    \caption{
    (Left) Comparison of adaptation mechanisms. (a) In-weight learning updates some or all model parameters. (b) In-context learning (ICL) adapts behavior by conditioning on exemplar tokens without parameter updates. (c) \algoname (ours) learns an entire KV cache directly (trainable prefixes) at each attention layer, compressing demonstrations into a compact cached state while keeping the backbone frozen. (Right) Test accuracy (\%) of in-context and in-weight adaptation approaches on GPQA-Diamond; \algofullname dominates both in-context and in-weight baselines, also surpassing full supervised fine-tuning.\looseness=-1}
    \label{fig:teaser}
\end{figure}

This scaling has four immediate implications that collectively position learning in-context as a viable route for reasoning. 
First, LLMs can reliably perform challenging reasoning tasks \emph{without updating pretrained weights and without overloading the context}. Specifically, \algoname improves GPQA-Diamond by up to 11\%- and outperforms SFT and LoRA.
Second, \algoname is markedly more \emph{efficient}, achieving equal or better accuracy with 59\% less data and 46\% fewer trainable parameters than LoRA on GSM8K. 
Third, and most consequential for deployment, \algoname yields faster inference with higher accuracy, generating 34\% shorter reasoning chains while improving accuracy by 11\% over SFT on GPQA-Diamond and translating directly into lower inference compute and monetary expense.
Taken together, these results show that \algoname overcomes the scaling limits of exemplar-based ICL, enabling reasoning with frozen models while simultaneously improving data efficiency, reducing the trainable footprint and inference-time cost.

To explain these gains, we theoretically prove that \algoname can be \emph{strictly more expressive} than LoRA. Low-rank weight updates act through the input sequence, constraining the induced key-value cache by both the adapter rank and input rank, resulting in what we call a \emph{carrier bottleneck}. \algoname bypasses this by injecting learned key-value vectors directly into attention, enabling value-space directions that rank-limited updates cannot realize. We confirm this mechanism empirically, showing that, \algoname increases the effective rank of learned representation by $\sim$20\% relative to SFT and LoRA.\looseness=-1 

To summarize, the key contributions of our work are:
\begin{itemize}
    \item We show reasoning need not be learned through parameter updates. By distilling in-context data into a fixed key–value prefix, \algoname overcomes the scaling limits of exemplar-based ICL and matches or surpasses in-weight adaptation methods such as LoRA and SFT on challenging reasoning benchmarks (GSM8K, MATH, AIME, GPQA).
    \item \algoname improves efficiency along three axes without sacrificing accuracy: on GSM8K it needs 59\% less \emph{training data} and 46\% fewer \emph{trainable parameters} than LoRA, and at inference it boosts accuracy by 44\% with 90\% less compute than ICL; on GPQA it reduces \emph{generation length} by 34\% while improving accuracy by 11\% over SFT.\looseness=-1
    \item We prove \algoname can be strictly more expressive than LoRA due to an input-dependent bottleneck in low-rank weight updates, and confirm this empirically by showing a $\sim$20\% increase in representation's effective rank over SFT and LoRA.
\end{itemize}

\section{Prefix Tuning}
\label{sec:pt}
Our method, \textsc{ReasonCache}, builds on prefix tuning (PT)~\citep{li2021prefix}, which we formalize in this section. 
In-context learning (ICL) operates by constructing a key–value (KV) cache from the tokens provided in the context: the model processes these tokens through frozen weights, producing keys and values that subsequent tokens attend to. The expressiveness of ICL is therefore limited by what can be represented through the forward pass of raw text under pretrained weights. Moreover, the size of the KV cache is directly determined by the context length. In contrast, PT removes these limitations by making the KV cache itself a learnable object. Instead of deriving keys and values from input tokens, PT directly optimizes auxiliary KV vectors at each attention layer. This section formalizes the PT mechanism and establishes its relationship to ICL and prompt tuning.\\

\subsection{Attention and the KV-Cache}

Consider a transformer with $L$ layers and dimension $d$. Given the representation $H^{(\ell-1)} \in \mathbb{R}^{n \times d}$ from the previous layer (with $H^{(0)}$ denoting the input embeddings), layer $\ell$ computes queries, keys, and values:
\[
Q^{(\ell)} = H^{(\ell-1)} W_Q^{(\ell)}, \qquad
K^{(\ell)} = H^{(\ell-1)} W_K^{(\ell)}, \qquad
V^{(\ell)} = H^{(\ell-1)} W_V^{(\ell)}.
\]
The attention output in each layer is
\(
Y^{(\ell)} = \mathrm{softmax}\!\left( Q^{(\ell)} (K^{(\ell)})^\top / \sqrt{d} \right) V^{(\ell)}.
\)
During autoregressive generation, keys and values are cached so that each new token can attend to all previous tokens without recomputation. This \emph{KV-cache} $\mathcal{C} = \{(K^{(\ell)}, V^{(\ell)})\}_{\ell=1}^{L}$ serves as the model's working memory: it determines what information is accessible to subsequent tokens.
\subsection{The Mechanism of Prefix Tuning}
Prefix tuning introduces $m$ trainable key--value pairs $(P_K^{(\ell)}, P_V^{(\ell)}) \in \mathbb{R}^{m \times d} \times \mathbb{R}^{m \times d}$ at each layer $\ell$. The augmented keys and values are:
\[
\tilde{K}^{(\ell)} =
\begin{bmatrix}
P_K^{(\ell)} \\[2pt]
K^{(\ell)}
\end{bmatrix},
\qquad
\tilde{V}^{(\ell)} =
\begin{bmatrix}
P_V^{(\ell)} \\[2pt]
V^{(\ell)}
\end{bmatrix}.
\]
Each query $q_i^{(\ell)}$ attends jointly to prefix and token-derived keys, producing an output that is a convex combination over all values:
\[
\tilde{y}_i^{(\ell)} = \sum_{p=1}^{m} \alpha_{ip} \, (P_V^{(\ell)})_p + \sum_{j=1}^{n} \alpha_{i,m+j} \, v_j^{(\ell)},
\]
where $\alpha_{ip}, \alpha_{i,m+j} \geq 0$ are the attention weights (summing to one across each row). Throughout, all pretrained weights remain frozen; only the prefix parameters $\mathcal{P} = \{(P_K^{(\ell)}, P_V^{(\ell)})\}_{\ell=1}^{L}$ are optimized.

\remark The prefix values $(P_V^{(\ell)})_p$ are free parameters that can point in \emph{any} direction in $\mathbb{R}^d$. In contrast, token-derived values $v_j^{(\ell)}$ must lie in the rowspace of $H^{(\ell-1)} W_V^{(\ell)}$. This geometric freedom is the source of PT's enhanced expressivity, which we formalize in \Cref{sec:theory}.

\remark The prefix influences the model at two scales. Within each layer, prefix values contribute to the attention mixture. Across layers, the attention output (which includes contributions from the prefix) becomes the hidden representation from which subsequent keys and values are derived. Thus deeper layers see token representations that have already been shaped by prefix vectors at earlier layers.

\subsection{Relationship to ICL and Prompt Tuning}
Prefix tuning can be understood as a generalization of two related methods:

\begin{itemize}
    \item \textbf{In-context learning.} In ICL, demonstration tokens $X_{\text{demo}}$ are prepended to the input and processed through frozen weights, producing a demonstration cache $\mathcal{C}_{\text{demo}} = \{(K_{\text{demo}}^{(\ell)}, V_{\text{demo}}^{(\ell)})\}_{\ell=1}^{L}$. This is equivalent to PT with $P_K^{(\ell)} = K_{\text{demo}}^{(\ell)}$, $P_V^{(\ell)} = V_{\text{demo}}^{(\ell)}$, and no optimization.

    \item \textbf{Prompt tuning.} Prompt tuning \citep{lester2021power} learns continuous embeddings $E \in \mathbb{R}^{m \times d}$ prepended at the input layer \textit{only}, which then propagate through the frozen network to produce KV vectors at each layer. This is equivalent to PT where the prefix vectors are constrained to be outputs of the frozen model on $E$, rather than free parameters.
\end{itemize}

The key distinction is that PT optimizes KV vectors at each layer independently, bypassing the frozen projections. This gives PT strictly greater representational freedom as it can produce prefix vectors that no input embedding could generate.

\subsection{Training and Inference}
The prefix parameters $\mathcal{P} = \{(P_K^{(\ell)}, P_V^{(\ell)})\}_{\ell=1}^{L}$ are optimized to minimize the standard next-token prediction loss over a dataset $\mathcal{D} = \{D_1, \ldots, D_M\}$ of $M$ sequences:
\begin{equation}
\label{eq:pt_objective}
\min_{\mathcal{P}} \sum_{j=1}^{M} \mathcal{L}(D_j \mid \mathcal{P}),
\qquad
\mathcal{L}(D_j \mid \mathcal{P}) = -\sum_{t=1}^{|D_j|} \log p_\theta\bigl(x_t^{(j)} \mid x_{<t}^{(j)}, \mathcal{P}\bigr),
\end{equation}
where $\theta$ denotes the frozen model parameters.

\paragraph{Initialization.}
Prefixes can be initialized randomly or from demonstration KV-caches, truncated or padded to the prefix length $m$. Alternatively, a trainable auxiliary network $f_\phi:\mathbb{R}^{d}\!\to\!\mathbb{R}^{2Ld}$ (e.g., an MLP) can generate the prefix for a fixed random input; this reparameterization often stabilizes optimization~\citep{li2021prefix}. At inference, the prefix is computed once and cached, and $f_\phi$ can be discarded.

The formulation above presents PT as a general mechanism for injecting learned key–value vectors into the attention computation of a frozen transformer.
In the rest of this paper, we use \algoname to denote PT specialized to reasoning.
In other words, \textsc{ReasonCache} introduces no new adaptation primitive; rather, it studies and demonstrates how learned KV prefixes can scale in-context learning into a reliable and efficient mechanism for reasoning.\looseness=-1

\section{Experimental Results}

Our experiments address the following key question:
\emph{Can reasoning capabilities traditionally learned via in-weight adaptation instead be acquired through in-context mechanisms, without modifying pretrained model parameters and without overloading the context?} We answer this question by evaluating \algoname across two primary dimensions: (1) \emph{accuracy} on challenging reasoning benchmarks, comparing against both in-context and in-weight baselines (\Cref{sec:main results}); and (2) \emph{efficiency} across three axes: data efficiency (\Cref{sec:sample efficiency}), inference efficiency (\Cref{sec:inference efficiency}), and parameter efficiency (\Cref{sec:parameter efficiency}). Ablation studies on design choices are deferred to~\Cref{sec:ablations}.

\subsection{Experimental Setup and Datasets}
\textbf{Models and Datasets.} We evaluate methods across both short- and long-form reasoning tasks. 
For short-reasoning tasks, we adapt a LLaMA-2 on MetaMathQA~\citep{yu2023metamath} and evaluate on GSM8K~\citep{cobbe2021training} and MATH~\citep{hendrycks2020measuring}. For long-form reasoning, we adapt a Qwen-2.5-7B-Instruct~\citep{yang2024qwen2} on a filtered subset of OpenThoughts-3~\citep{guha2025openthoughts} and evaluate on GPQA-Diamond~\citep{rein2024gpqa} and AIME ~\citep{MAA2024AIMEI,MAA2025AIMEI}.
We construct the OpenThoughts-3 subset by retaining only examples whose reasoning traces fit within a 4096-token budget, matching the controlled-context setting of prior work~\citep{guha2025openthoughts}.
At inference, we cap model ``thinking'' to 4096 tokens via the decoding-time intervention of Muennighoff et al.~\citep{muennighoff2025s1}. Concretely, we force termination of the reasoning phase by appending the end-of-thinking delimiter and \emph{``Final Answer:''}, prompting the model to output its current best solution under the imposed constraint. For more details, refer to~\Cref{sec:evaluation}. Across all datasets, we follow the standard train/eval splits provided by the Language Model Evaluation Harness (\texttt{lm-evaluation-harness})~\citep{eval-harness}.

\textbf{Baselines.} We compare \algoname against standard \emph{in-weight} baselines (supervised fine-tuning (SFT) and LoRA) and \emph{in-context} baselines (exemplar-based in-context learning (ICL) and prompt tuning). For LoRA, we apply adapters to the attention key-value projections and sweep ranks $r\in\{2^k\}_{k=0}^7$.
For \algofullname and prompt tuning, we sweep the number of virtual tokens $m\in\{2^k\}_{k=0}^{10}$.

\textbf{Training and Evaluation Setup.} We train all methods with AdamW and a cosine learning-rate schedule (warmup ratio $0.05$), using zero weight decay. On OpenThoughts-3, we train for 13 epochs with batch size 32 and sequence length 8192; on MetaMathQA, we train for 3 epochs with batch size 128 and sequence length 2048. At inference, we evaluate with \texttt{lm-evaluation-harness} under greedy decoding (temperature $0$) and report exact-match accuracy (\texttt{pass@1}). Unless stated otherwise, methods share the same base model, tokenizer, training examples, optimization budget, and decoding configuration, and we select hyperparameters by best validation performance before reporting final test accuracy.

\begin{figure}[!htb]
    \centering
    \includegraphics[width=1.\linewidth]{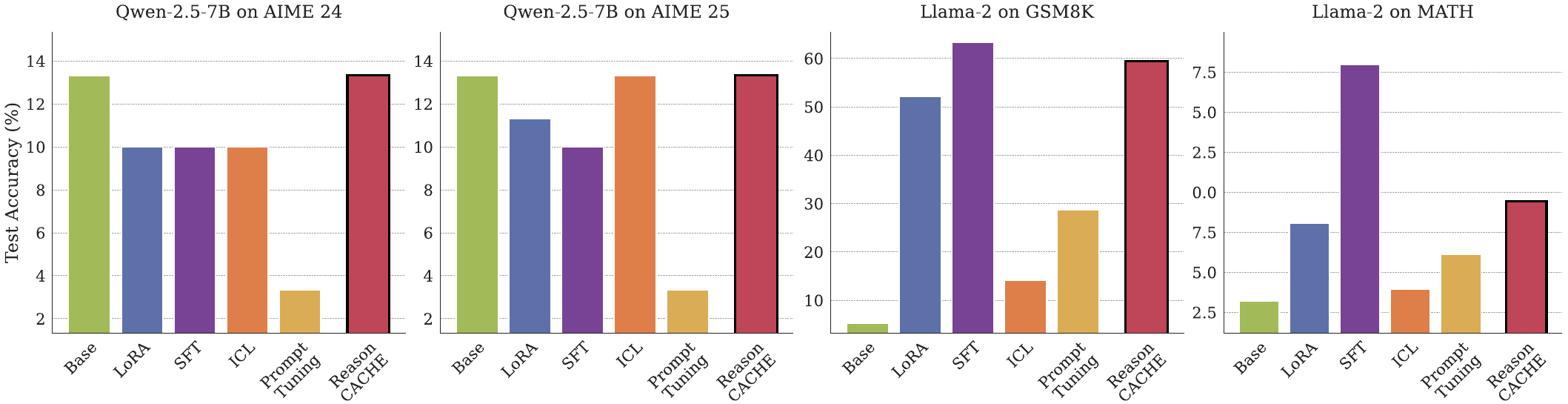}
    \caption{Test accuracy (\%) of in-context and in-weight adaptation methods across AIME 24/25, GPQA-Diamond (shown in~\Cref{fig:teaser}), GSM8K and MATH. 
    \algoname consistently outperforms in-context baselines (ICL, Prompt Tuning) and exceeds LoRA while keeping pretrained weights frozen. On GPQA-Diamond, \algoname surpasses full supervised fine-tuning (SFT).}
    \label{fig:results_main_fig}
\end{figure}

\subsection{\algofullname Enables Effective Reasoning in LLMs Without Modifying Weights}\label{sec:main results}
\textbf{\algofullname Scales In-Context Learning:} \Cref{fig:teaser}(right) and \Cref{fig:results_main_fig} summarize final test accuracy across all benchmarks. We first examine whether \algoname meaningfully extends the capabilities of standard ICL. We observe a clear limitation of \emph{shallow} contextual adaptation: both ICL and prompt tuning yield only modest gains, particularly on tasks requiring extended reasoning. In contrast, \algofullname consistently outperforms other in-context methods.

\textbf{\algofullname Outperforms In-Weight Adaptation on Reasoning Benchmarks.}
We next compare \algofullname against IWL methods. Even under matched parameter budgets, \algoname matches or beats LoRA and even full SFT (on GPQA-Diamond). Specifically, when trained on OpenThoughts-3, \algoname achieves the best performance on GPQA-Diamond ($41.92\%$), exceeding both LoRA and full SFT.

This advantage, however, does not seem to extend to AIME 24 and 25: all methods remain near the baseline ($\sim$13\%). IWL even hurts performance ($\sim$10\%), while \algoname preserves baseline accuracy. We attribute this to our 4096-token filtering, which likely underrepresents the extended reasoning that competition mathematics demands.\looseness=-1

Having established that \algoname matches or exceeds IWL methods in absolute accuracy, we next evaluate efficiency across three dimensions: how much \emph{data} is required during learning (data efficiency), how much \emph{compute} is needed at inference (inference efficiency), and how many \emph{parameters} must be stored (parameter efficiency). Across all three, \algoname exhibits superior accuracy-cost tradeoffs: it uses 59\% less data than LoRA to reach the same accuracy, achieves 90\% lower inference compute than ICL at higher accuracy, and requires 46\% fewer parameters than LoRA while maintaining better performance. 

\subsection{\algofullname is Data Efficient: Pushes the Accuracy vs Data Frontier}\label{sec:sample efficiency}

\noindent
\begin{minipage}[t]{0.52\columnwidth}
\vspace{0pt} 
\setlength{\parskip}{0pt} 
\algoname combines ICL’s data efficiency with the scalability of IWL methods. In \Cref{fig:data_efficiency_plot}, \algoname traces the Pareto frontier of accuracy versus training dataset size across nearly four orders of magnitude. In the low-data regime, \algoname matches or exceeds ICL. Initializing \algoname from few-shot exemplars preserves the inductive bias of ICL, and subsequent optimization enables further gains beyond the initial few-shot performance. In contrast, however, ICL degrades as more examples are added (longer contexts dilute attention over exemplars and exacerbate position effects), consistent with reported long-context issues in the literature~\citep{agarwal2024many,zhang2025more,zhang2025memory,liu2024lost}.\\
\par IWL methods exhibit the opposite failure mode: LoRA and SFT tend to overfit with limited data (often mitigated by heavy augmentation~\citep{yang2024synthetic,eyuboglu2025cartridges} or large-scale training). \algoname is
\end{minipage}\hfill%
\begin{minipage}[t]{0.45\columnwidth}
\vspace{0pt} 
\centering
\includegraphics[width=\linewidth]{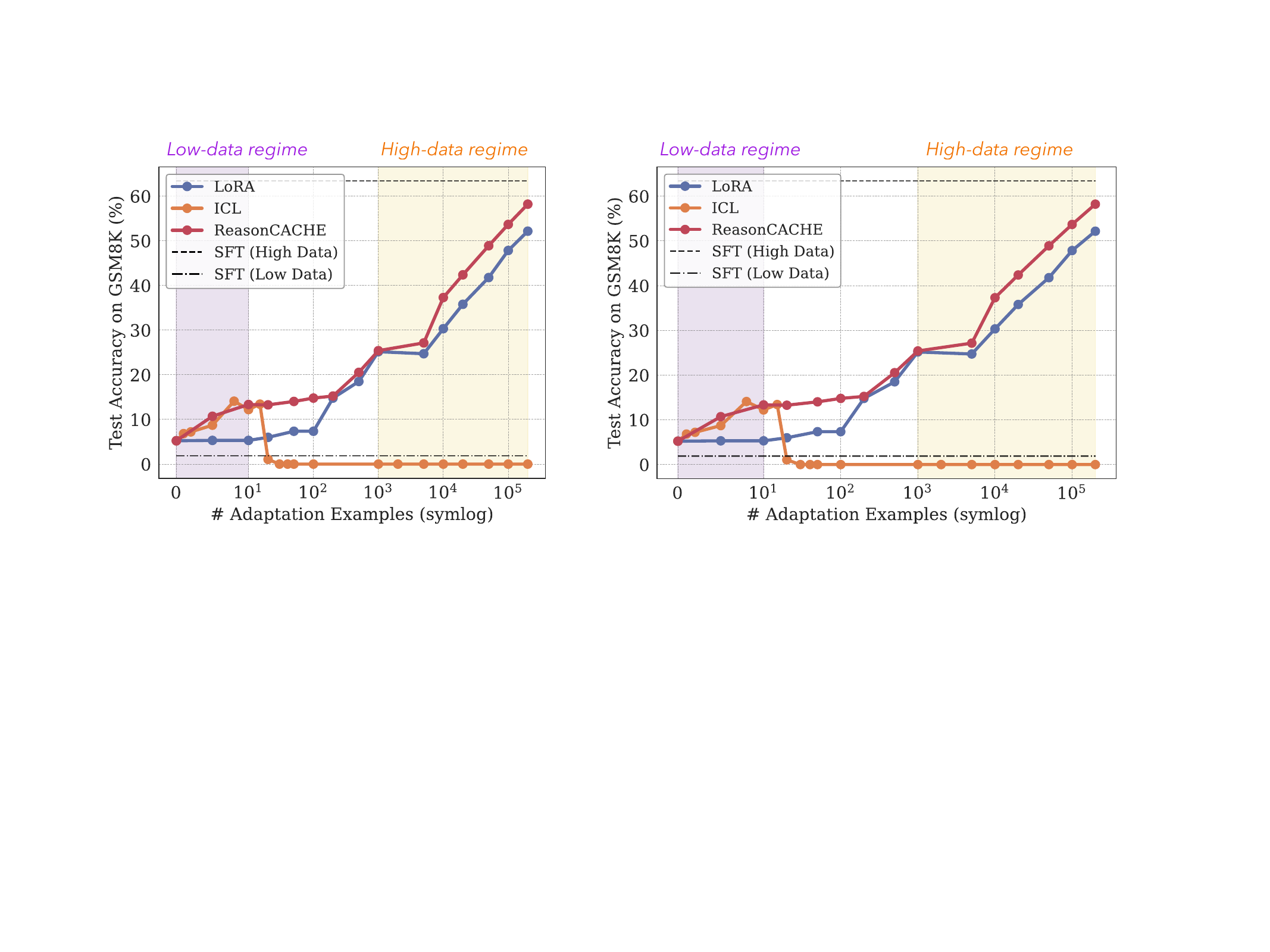}
\captionof{figure}{Accuracy as a function of training examples. \algoname matches ICL’s data efficiency at low data and continues to scale with additional data like in-weight methods.\looseness=-1}
\label{fig:data_efficiency_plot}
\end{minipage}

$\sim 59$\% more data-efficient than LoRA (to reach 50\% accuracy), consistently outperforming it throughout this regime and closing the gap to high-data SFT, despite keeping the pretrained backbone frozen. This places \algoname uniquely on the ICL–IWL spectrum: it inherits \emph{ICL’s data efficiency} and inductive bias via few-shot initialization, while gaining the \emph{data-scaling benefits of IWL methods} through gradient-based optimization. \algoname therefore can be a bridge between in-context and in-weight adaptation.\looseness=-1

\subsection{\algofullname is Efficient At Inference: Pushes the Accuracy vs Compute Frontier}\label{sec:inference efficiency}

\begin{figure}[!htb]
\centering
\begin{minipage}{0.3\textwidth}
\centering
\includegraphics[width=1.\textwidth]{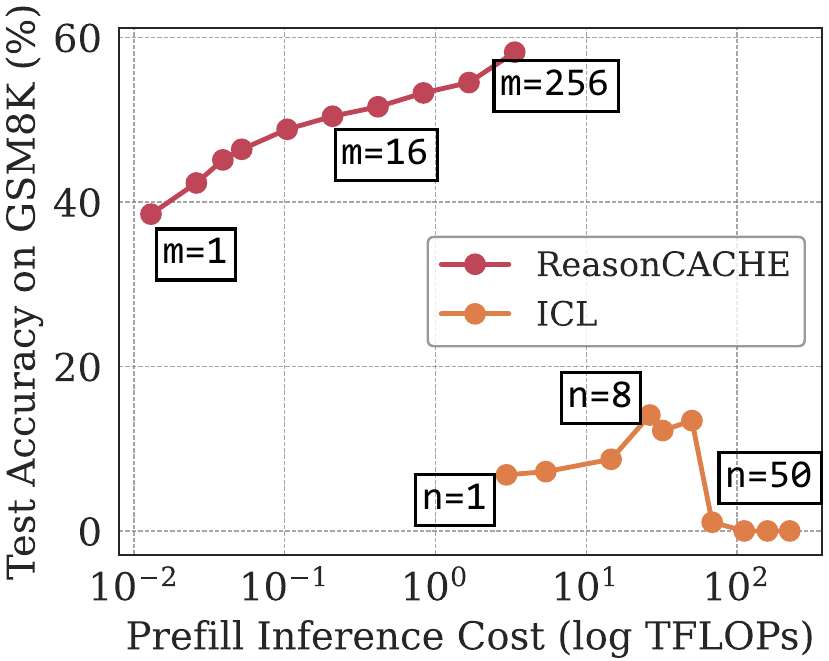}
\end{minipage}
\hfill
\begin{minipage}{0.3\textwidth}
\centering
\includegraphics[width=1.\textwidth]{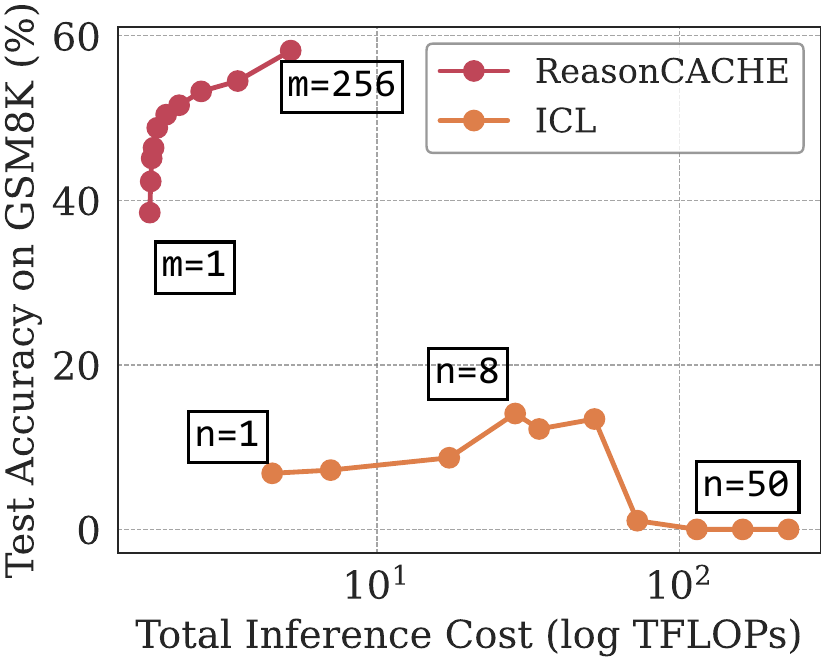}
\end{minipage}
\hfill
\begin{minipage}{0.34\textwidth}
\caption{Accuracy as a function of inference cost measured in TFLOPs. \textit{Left}: Prefill cost only. \textit{Right}: Total inference cost (prefill + decoding). \algofullname consistently dominates in-context learning, achieving better accuracy at substantially lower inference compute across both metrics. $n$ and $m$ denote the number of in-context example, and the number of tokens allocated to for \algofullname, respectively.}
\label{fig:inference_cost}
\end{minipage}
\end{figure}

A fundamental advantage of \algoname is that it decouples the learning budget (i.e., number of adaptation examples) from the inference cost. In ICL, the adaptation data must be included in the context at inference. Even a single long reasoning example (e.g., 10k tokens) directly increases the prefill cost quadratically and decoding cost linearly. \algoname removes this coupling by learning from hundreds of thousands of tokens, while using only a short compact learned prefix (length $m$) at inference. Since prefill cost scales quadratically with context length, replacing long exemplar prompts with a short prefix yields substantial savings.\footnote{For prompt length $S$ and generation length $T$, prefill cost is $O(S^2)$ and decoding cost is $O(T(S+T))$.}

\textbf{\algofullname Reduces Prefill Cost.} We first consider GSM8K, where short reasoning chains make inference cost dominated by prefill. \Cref{fig:inference_cost} reports accuracy as a function of inference compute (measured in TFLOPs) for prefill (left) and prefill+decoding (right) on GSM8K. \algoname dominates ICL, outperforming the best ICL configuration by 44.8 points while using 90\% less total inference compute. This gap is driven by prefill costs, since ICL incurs quadratic overhead as demonstrations are added, whereas \algoname replaces them with a short learned prefix. We note that IWL methods incur no additional prefill overhead at all since they modify weights rather than context. However, as shown next, \textsc{ReasonCache}’s shorter generations offset its modest prefill cost, yielding better overall inference efficiency.\looseness=-1

\begin{figure}[!htb]
    \centering
    \includegraphics[width=1.\linewidth]{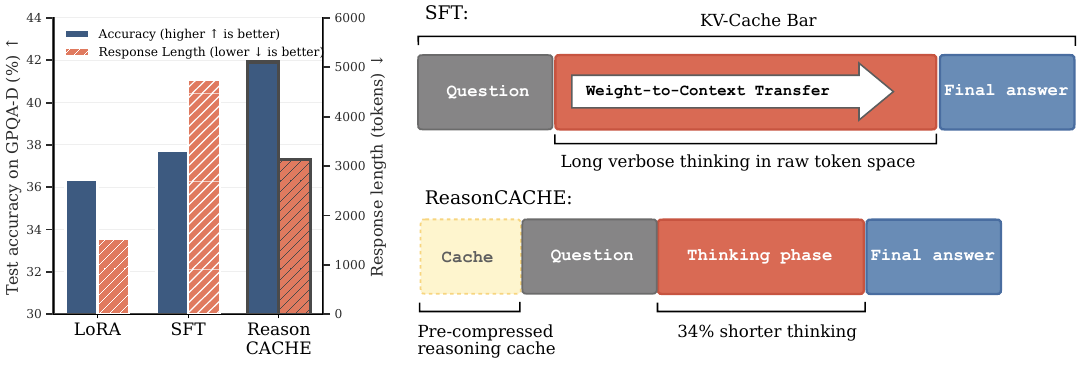}
    \caption{ (Left): On GPQA-Diamond, \algoname achieves 11\%
  higher accuracy than SFT while generating 34\% shorter responses. LoRA produces the shortest responses but at the cost of lower
  accuracy. (right): We hypothesize that SFT externalizes weight-encoded procedural knowledge into explicit context during
  generation. Although the model has internalized reasoning strategies through weight updates, it nonetheless regenerates these
  patterns as explicit token sequences at inference time, leading to unnecessary verbosity. Our conjecture is that in contrast, \algoname stores
  procedural knowledge directly in the KV cache, eliminating the need for explicit externalization.}
    \label{fig:verbosity}
\end{figure}

\textbf{\algofullname Reduces Generation Length.}
On long-form reasoning benchmarks such as GPQA-Diamond, models often generate thousands of reasoning tokens, causing decoding to dominate the inference cost.
Interestingly, \algoname attains higher accuracy while producing substantially shorter generations (\Cref{fig:verbosity}). Compared to SFT, \algoname reduces generation length by 34\% while improving accuracy by about $11\%$. LoRA produces shorter outputs but underperforms noticeably on this benchmark. 
We hypothesize that the learned prefix likely encodes task guidance; instead of spending tokens to re-derive the approach step-by-step, the model conditions on this information during the forward pass and can proceed more directly to the answer. In-weight methods, by contrast, must “unpack” their learned knowledge from weights to the context as shown in~\Cref{fig:verbosity} (right). As we formalize in \Cref{sec:theory}, weight updates influence the KV-cache only through the input tokens, so exploiting learned knowledge requires sufficiently rich context in input.

\subsection{\algofullname is Parameter Efficient: Pushes Accuracy vs Parameter Frontier}\label{sec:parameter efficiency}

\begin{minipage}[t]{0.58\columnwidth}
\vspace{0pt} 
\setlength{\parskip}{0pt} 
Finally, we examine \emph{parameter} efficiency, which captures the cost of \emph{storing} and transferring an adaptation once it has been learned.
Under matched parameter budgets, \algoname consistently dominates LoRA across the entire accuracy–parameter curve on GSM8K (\Cref{fig:acc_vs_params}). For example, to reach a target test accuracy of 50\%, \algoname requires 46\% fewer parameters than LoRA.
\algoname is more parameter-efficient in the low-budget regime and continues to improve as the budget increases, while LoRA saturates early.
This trend aligns with our theoretical analysis in~\Cref{sec:theory}: by injecting learned key--value prefixes at every layer, \algoname can steer attention along directions that low-rank weight updates cannot realize even at a \emph{higher} parameter budget.\\

A note on parameter counting: When \algoname uses the MLP reparameterization (see~\Cref{sec:ablations}), the MLP serves only to stabilize training and is discarded at deployment, leaving only the generated prefix vectors; we therefore report only the deployed parameters (the prefix vectors) in \Cref{fig:acc_vs_params}.
\end{minipage}\hfill%
\begin{minipage}[t]{0.38\columnwidth}
\vspace{0pt} 
\centering
\includegraphics[width=\linewidth]{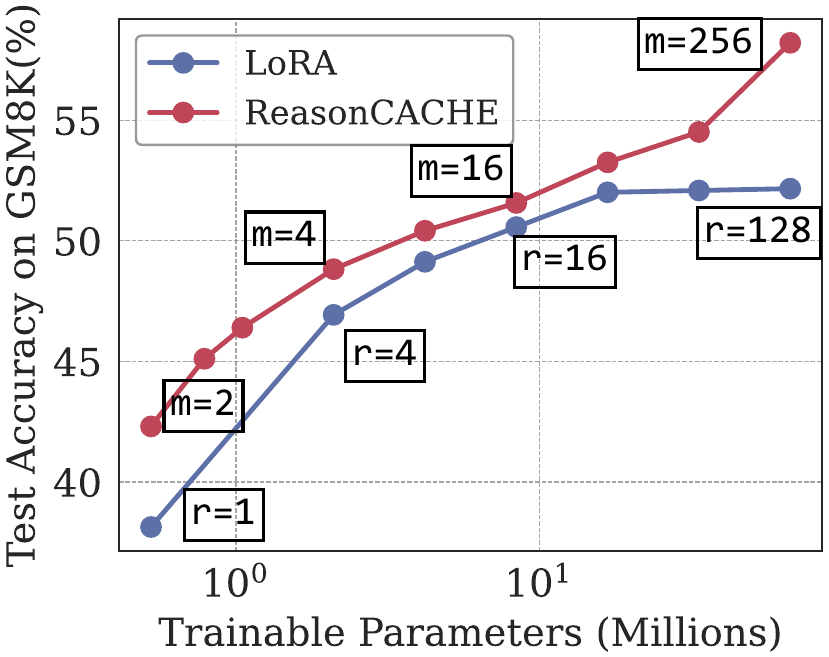}
\captionof{figure}{\algoname achieves consistently higher test accuracy at matched parameter budgets and continues to improve where LoRA saturates.\looseness=-1}
\label{fig:acc_vs_params}
\end{minipage}

\section{Theoretical Perspectives}
\label{sec:theory}

The preceding experiments demonstrate that \algofullname can match or exceed its in-weight counterpart, LoRA, on challenging reasoning benchmarks (\Cref{sec:main results}). 
What explains this?
In this section, we show that \algofullname (and PT broadly) and LoRA are governed by fundamentally different expressivity bottlenecks: each method can, in certain regimes, access subspaces in value space that are unreachable by the other. Since this analysis is not specific to reasoning, we present it in the language of \emph{Prefix Tuning} (PT) for clarity, using PT notation throughout. We characterize exactly when each method is more expressive and provide geometric intuition for the comparison. All proofs are deferred to \Cref{sec:proofs}.

\subsection{Setting and Notation}

We adopt the notation of \Cref{sec:pt}. Consider a single attention layer with frozen value projection $W_V \in \mathbb{R}^{d \times d}$. Given context embeddings $X \in \mathbb{R}^{n \times d}$, the base model produces value vectors $v_1, \ldots, v_n$ where $v_i = x_i W_V$. We define 
\[
S_X := \mathrm{span}\{v_1, \ldots, v_n\} \subseteq \mathbb{R}^d,
\]
which captures all directions already expressible by the base model on this context. Let $\Pi_X$ denote the orthogonal projector onto the complement $S_X^\perp$, and define:
\[
t_X := \rank(X), \qquad \nu_X := \dim S_X^\perp = d - \dim S_X.
\]
Here, $t_X$ is the rank of the input sequence and $\nu_X$ as the \emph{novelty capacity}, defined as the dimension of the subspace orthogonal to what the base model can already express. Intuitively, $\nu_X$ quantifies the remaining ``headroom'' in representation space for adding directions that are not already induced by the input context.

A LoRA update modifies the value projection to $W_V' = W_V + \Delta_V$ with $\rank(\Delta_V) \le r$, producing updated values $v_i' = x_i(W_V + \Delta_V)$. Prefix tuning instead contributes learned prefix values $P_V \in \mathbb{R}^{m \times d}$, which are concatenated with the token-derived values. Since the attention output is a convex combination of value vectors, it cannot produce directions outside the span of those values. The key comparison is therefore between the novelty subspaces, i.e., the new directions each method can add beyond $S_X$:
\[
\Pi_X \,\mathrm{span}\{v_1', \ldots, v_n'\} \quad \text{(LoRA)} \qquad \text{vs.} \qquad \Pi_X \,\mathrm{span}\{(P_V)_1, \ldots, (P_V)_m\} \quad \text{(PT)}.
\]
We refer to these subspaces as \emph{novelty subspaces} since they capture directions that are not realizable by the base model under the given context.

\subsection{Characterizing Achievable Novelty Subspaces}

We first characterize exactly which novelty subspaces each method can realize.

\begin{restatable}{proposition}{loraprop}
\label{prop:lora}
(LoRA Novelty Subspaces). At a fixed context $X$, a subspace $U \subseteq S_X^\perp$ is realizable as $\Pi_X \,\mathrm{span}\{v_1', \ldots, v_n'\}$ for some value update $\Delta_V$ with $\rank(\Delta_V) \le r$ if and only if
\[
\dim U \le \min\{t_X, r\}.
\]
\end{restatable}

\begin{restatable}{proposition}{ptprop}
\label{prop:pt}
(Prefix Tuning Novelty Subspaces). At a fixed context $X$, a subspace $U \subseteq S_X^\perp$ is realizable as $\Pi_X \,\mathrm{span}\{v_1', \ldots, v_n'\}$ for some value update $\Delta_V$ with $\rank(\Delta_V) \le r$ if and only if
\[
\dim U \le \min\{t_X, r\}.
\]
\end{restatable}

\Cref{prop:lora} and \Cref{prop:pt} state that both methods can realize \emph{arbitrary} sets of new directions outside the base model’s context-induced span; the difference is the bottleneck on how many such directions they can realize. For LoRA, the number of independent new directions is jointly limited by (i) how many independent directions the input context provides and (ii) the adapter’s rank. For prefix tuning, the directions are likewise unconstrained, but their number is limited only by the number of prefix vectors.

\begin{remark}
 \Cref{prop:lora} reveals that the context $X$ creates a bottleneck at $\min\{t_X, r\}$ for LoRA, resulting in what we call \emph{carrier bottleneck}. Even with arbitrarily large adapter rank $r$, LoRA cannot produce novelty subspaces of dimension exceeding the context rank $t_X$. In contrast, \Cref{prop:pt} shows that PT is \emph{prefix-limited} only by $m$, the number of prefix vectors, which contribute directly without passing through the context. 
\end{remark}

\subsection{Characterizing the Expressivity of LoRA and Prefix Tuning}

We now compare the entire families of achievable novelty subspaces. Define:
\begin{align*}
\mathcal{L}(r) &:= \{\Pi_X \,\mathrm{span}\{v_1', \ldots, v_n'\} : \rank(\Delta_V) \le r\}, \\
\mathcal{P}(m) &:= \{\Pi_X \,\mathrm{span}\{(P_V)_1, \ldots, (P_V)_m\} : P_V \in \mathbb{R}^{m \times d}\}.
\end{align*}

\begin{restatable}{theorem}{mainthm}
\label{thm:main}
(Expressivity Comparison). At a fixed context $X$, define $D_{\text{LoRA}} := \min\{t_X, r, \nu_X\}$ and $D_{\text{PT}} := \min\{m, \nu_X\}$. Then:
\begin{enumerate}
    \item[(i)] $\mathcal{L}(r) \subseteq \mathcal{P}(m)$ if and only if $D_{\text{LoRA}} \le D_{\text{PT}}$.
    \item[(ii)] $\mathcal{P}(m) \subseteq \mathcal{L}(r)$ if and only if $D_{\text{PT}} \le D_{\text{LoRA}}$.
\end{enumerate}
Consequently, strict inclusion $\mathcal{L}(r) \subset \mathcal{P}(m)$ holds if and only if $D_{\text{LoRA}} < D_{\text{PT}}$ and equality $\mathcal{L}(r) = \mathcal{P}(m)$ holds if and only if $D_{\text{LoRA}} = D_{\text{PT}}$ and vice versa.
\end{restatable}

\begin{figure}[!htb]
    \centering
    \includegraphics[width=0.8\linewidth]{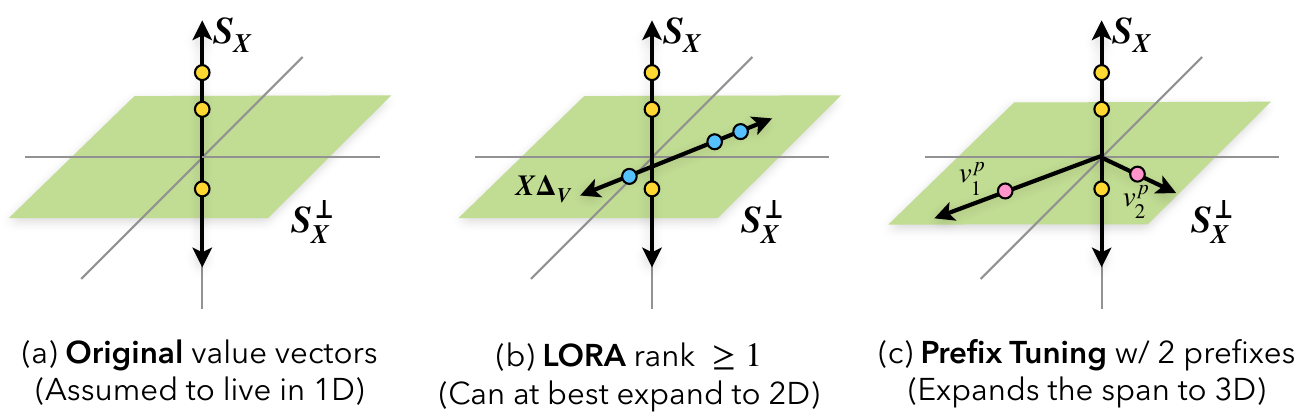}
    \caption{Prefix Tuning can span directions unreachable by LoRA. Assume the context $X$ has rank one ($t_X = 1$). The original value vectors span a 1D subspace $S_X$ (vertical axis); any new information must lie in the orthogonal novelty space $S_X^\perp$ (horizontal plane). {(a)} The base model is confined to $S_X$. {(b)} A LoRA update introduces novelty via $X \Delta_V$, but is carrier-limited by the context's rank, adding at most one new dimension regardless of the adapter rank $r$. {(c)} Prefix tuning directly injects $m$ learnable vectors. With $m = 2$, it can span the full 2D novelty space $S_X^\perp$, which LoRA cannot reach. This separation is not specific to the example: whenever $t_X < m$, PT accesses strictly more directions than LoRA.\looseness=-1}
    \label{fig:pt-expansion}
\end{figure}

\textbf{Interpretation.}
\Cref{thm:main} reveals that the expressive power of each method is governed merely by the prefix length $m$ and the LoRA bottleneck $\min\{t_X, r\}$, subject to the novelty capacity $\nu_X$.

\emph{PT is strictly more expressive} when $m > \min\{t_X, r\}$. This occurs in two scenarios: (i) when a low-rank context carrier limits the LoRA update, regardless of the adapter rank $r$; or (ii) when a low-rank adapter is chosen for parameter efficiency, and its rank is smaller than the prefix length. In these regimes, PT can realize novelty subspaces that no LoRA configuration can match.  

Note that, in contemporary PEFT practice, prefix lengths are often chosen to be moderately large, ranging from tens to thousands of tokens (e.g., $m\in[16,4096]$). LoRA however typically employs small adapter ranks most commonly $r\in[4,128]$). As a result, for sufficiently diverse and long inputs, the condition $m>\min\{t_X,r\}$ frequently holds, placing many practical deployments in the regime where prefix tuning enjoys a strict expressivity advantage.

\emph{LoRA is strictly more expressive} when $\min\{t_X, r\} > m$, that is, when a high-rank context combined with a sufficiently large adapter rank provides a wider bottleneck than the prefix length.

\emph{Both methods are equivalent} when the novelty capacity $\nu_X$ is the binding constraint. If $\nu_X < m$ and $\nu_X < \min\{t_X, r\}$, both methods saturate at $D_{\text{LoRA}} = D_{\text{PT}} = \nu_X$, and the choice reduces to other considerations such as computational efficiency.

\textbf{Connection to Chain-of-Thought Reasoning.} A useful implication of our analysis is that the parameter $t_X$ depends only on the \emph{context} available to the attention layer, regardless of whether that context comes from the original prompt or from tokens generated by the model itself. Producing a chain-of-thought expands this context with additional tokens, and can therefore increase $t_X$, enlarging the value span $S_X$ and the set of directions available for subsequent computation. Through \Cref{prop:lora}--\Cref{thm:main}, this offers a plausible geometric view of why chain-of-thought often improves accuracy~\citep{wei2021finetuned}; by increasing $t_X$, it relaxes the input bottleneck that limits how many new directions LoRA can induce on a given input and expands the attainable output space. 

\vspace{5mm}

\begin{corollary}[PT Strictly More Expressive than QK-LoRA]
\label{cor:qk}
For LoRA applied only to query and key matrices, the value update rank is $r = 0$, giving $D_{\text{LoRA}} = 0$. If the novelty space is not saturated ($\nu_X > 0$) and at least one prefix token is used ($m \ge 1$), then $D_{\text{PT}} \ge 1 > D_{\text{LoRA}}$, implying $\mathcal{L}(0) = \{\{0\}\} \subset \mathcal{P}(m)$.
\end{corollary}

This corollary highlights a fundamental difference in mechanism: LoRA on Q/K matrices can only \emph{re-weight} existing value vectors within their original span $S_X$, whereas prefix tuning can \emph{introduce entirely new vectors} into the attention mechanism's value dictionary, enabling outputs in directions previously unavailable.

\subsection{Empirical Evidence for Expressivity via Representation Geometry}
We empirically validate our analysis above by probing the representation geometry (details in \Cref{sec:validate_theory}). 
First, across attention heads, the pretrained model’s value vectors occupy a low-dimensional subspace, revealing substantial unused headroom in value space (\Cref{fig:effective_dimension}). 
\algoname exploits this headroom since learned prefix values place most of their energy (typically 70--90\%) outside the dominant value subspace, injecting directions the base computation scarcely uses rather than merely reweighting it (\Cref{fig:prefix_energy}). 
These directions further persist downstream: \algoname'd models exhibit higher effective rank in last-layer representations than both supervised fine-tuning and LoRA, corroborating our layer-level expressivity mechanism despite the intervening nonlinearities across layers.

\begin{figure}[!htb]
    \centering
    \begin{minipage}{0.30\textwidth}
    \centering
    \includegraphics[width=\textwidth]{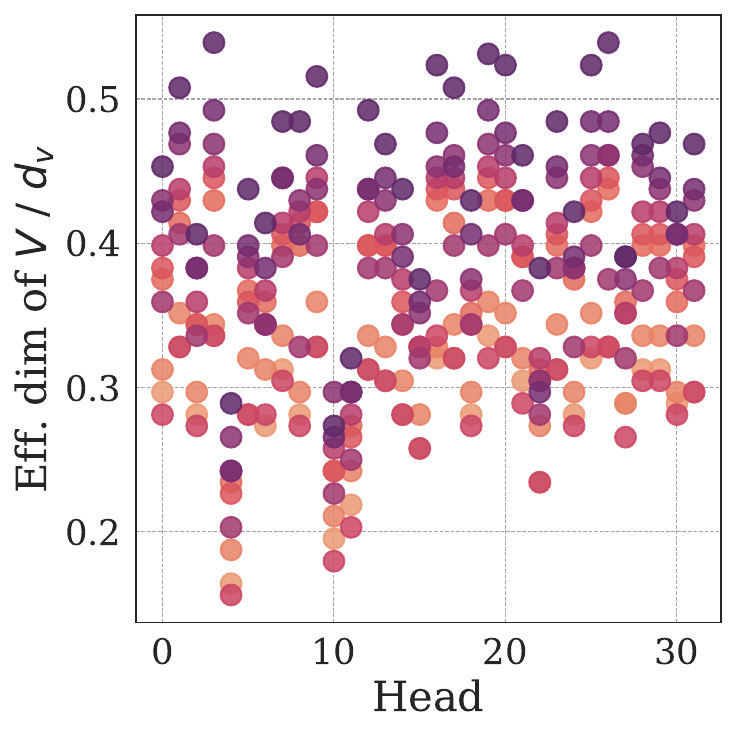}
    \caption{Effective dimension of value vectors across heads. Most heads use less than half of their available dimension, showing that the base model operates in a compact subspace.}
    \label{fig:effective_dimension}
    \end{minipage}
    \hfill
    \begin{minipage}{0.30\textwidth}
    \centering
    \includegraphics[width=\textwidth]{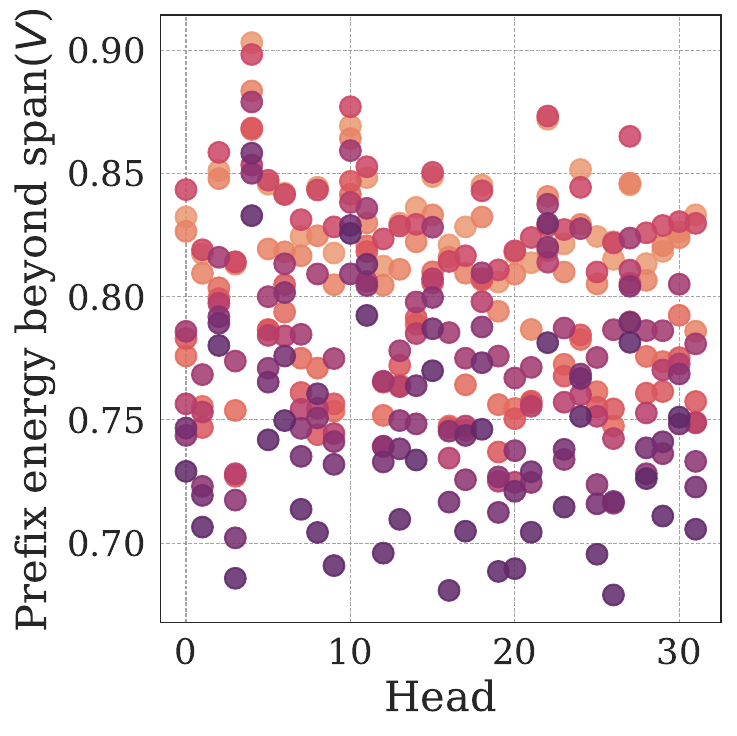}
    \caption{Fraction of prefix energy outside the value subspace of the base model. Prefixes place 80\% of their mass beyond this span, exploiting directions unused by the base model.\looseness=-1}
    \label{fig:prefix_energy}
    \end{minipage}
    \hfill
    \begin{minipage}{0.35\textwidth}
    \centering
    \includegraphics[width=\textwidth]{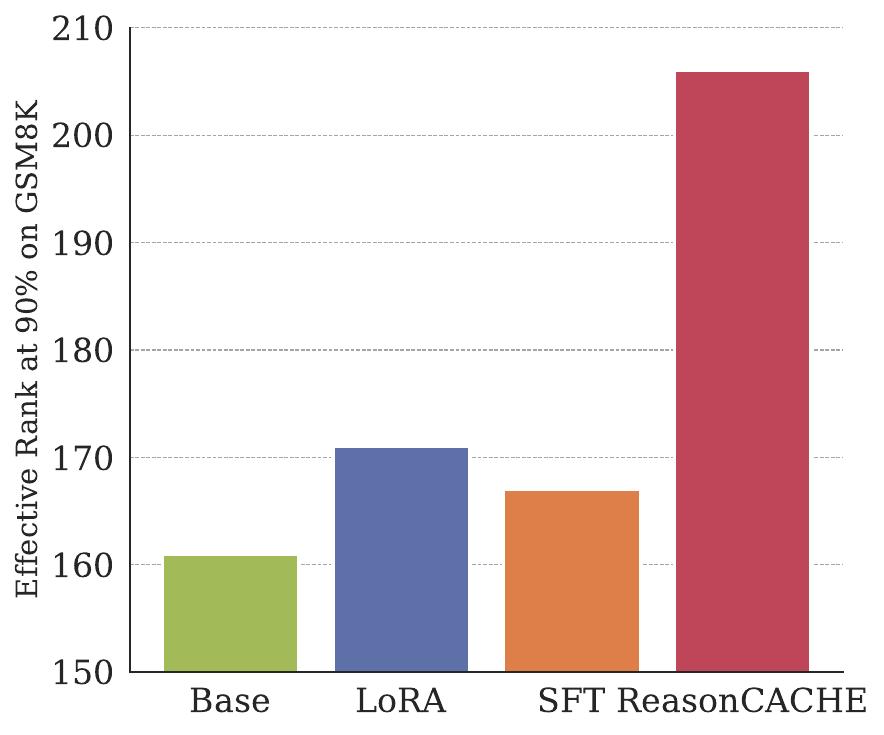}
    \caption{Effective rank of last-layer token representations, defined as the rank at which the cumulative singular-value mass reaches $90\%$. \algoname attains the highest effective rank among all methods.}
    \label{fig:effective_rank}
    \end{minipage}
\end{figure}

\subsection{An Illustrative Example}
\label{sec:learnability}

We now provide an example with a geometric interpretation comparing PT and LoRA, as shown in \Cref{fig:pt-expansion}. Consider a single attention layer, without nonlinearities or residual connections. Suppose the context $X$ has rank one, so all value vectors lie on a line i.e. $S_X$ is 1-dimensional.

Now consider a binary classification task: each token $i$ belongs to one of two classes, and the target output $y_i^\star$ should be $e_1$ for class 1 and $e_2$ for class 2, where $e_1, e_2$ are orthogonal unit vectors in $S_X^\perp$. Since the targets span two dimensions in $S_X^\perp$, any method that can only produce one-dimensional outputs will fail.

\textbf{Why QK-LoRA fails.} Adapting only the query and key matrices changes the attention weights but not the values. The output remains a convex combination of the original values $v_1, \ldots, v_n$, all of which lie in $S_X$. No matter how the attention is reweighted, the outputs cannot escape this 1D subspace. They thus have zero component in $S_X^\perp$, let alone the required two dimensions.

\textbf{Why QKV-LoRA fails.} Adding an update $\Delta_V$ to the value projection produces new values $v_i' = x_i(W_V + \Delta_V)$. Even if the LoRA rank $r$ is arbitrarily large, the product $X \Delta_V$ has rank at most $\rank(X) = 1$. This is because context $X$ limits what the weight update can express (\Cref{prop:lora}). The updated values span only a 1D subspace, which may have a component in $S_X^\perp$ but cannot fully span the two dimensions.

\textbf{Why \algoname succeeds.} Prefix tuning with $m = 2$ can place two prefix values $(P_V)_1 = e_1$ and $(P_V)_2 = e_2$ pointing in the two target directions. The prefix keys $(P_K)_1$ and $(P_K)_2$ can then be learned to route each token to the appropriate prefix value based on its class. This achieves the desired two-class separation without any constraint from the context's rank.

\vspace{5mm}

\begin{restatable}{theorem}{hierarchythm}
\label{thm:hierarchy}
(Loss Hierarchy). In the setting above, QK-LoRA and QKV-LoRA both incur strictly positive loss, while prefix tuning with $m = 2$ drives the loss to zero.
\end{restatable}

The formal proof, which constructs explicit solutions and lower bounds, is given in \Cref{sec:proofs}.

\section{Related Work}
\label{sec:related_work}

\emph{An extended discussion of related work appears in \Cref{sec:related_work_appendix}; here we summarize the most relevant threads.}

In-context learning (ICL) enables LLMs to adapt from demonstrations without gradient updates~\citep{brown2020language}, but struggles to scale: performance saturates or degrades as context length increases~\citep{agarwal2024many,liu2024lost,hong2025contextrot}, and ICL has been argued to be a shallow learner~\citep{de2025context}. This is particularly limiting for reasoning, where a single training example can exceed 10k tokens, making it infeasible to fit many demonstrations in context. The dominant approach for complex reasoning instead relies on weight updates via fine-tuning or reinforcement learning~\citep{guo2025deepseek,reasoningopenai,yu2023metamath}, but in-weight learning has its own limitations, such as the reversal curse~\citep{berglund2023reversal}.

Prefix tuning~\citep{li2021prefix} offers an alternative by learning continuous vectors that steer the model without weight updates. While originally studied for classification and generation, we explore whether prefix tuning can acquire complex reasoning skills. Related work on context compression~\citep{mu2023learning,chevalier2023adapting,ge2023context} and KV-cache distillation~\citep{caccia2025training} focuses on compressing contextual information rather than learning persistent skills. Cartridges~\citep{eyuboglu2025cartridges} compresses massive contexts into learnable KV-caches for knowledge retrieval; in contrast, we target skill acquisition (e.g., reasoning).

We view the KV-cache as a learnable \emph{memory interface} spanning a spectrum from raw tokens (recovering ICL) to gradient-optimized vectors (approaching in-weight methods). This framing connects to work on memory in neural networks~\citep{zhang2024memory,lin2025continual} and transformers as implicit optimizers~\citep{von2023transformers,mittal2025iterative}. Importantly, prefix tuning as studied here is \emph{not} a test-time learner: the prefix is trained offline and frozen at deployment. Developing methods where the KV-cache remains plastic at test time, enabling true continual learning, remains an open problem.

\section{Discussion}
In this work, we introduce \textsc{ReasonCache}, a mechanism that instantiates prefix tuning (PT) and learns a compact KV-prefix to distill and scale in-context data beyond the context window for efficient reasoning.
Without modifying any existing weights, post-training with \algoname can match or surpass in-weight adaptation methods such as LoRA and SFT on reasoning benchmarks (GSM8K, MATH, AIME, GPQA). \algoname turns out to be more data-efficient (59\% less data), produces more succinct reasoning chains (34\% shorter generations), and is substantially more parameter-efficient (46\% smaller adapted state) than in-weight methods.
We theoretically ground the benefits of \algoname by demonstrating that it can be more expressive by spanning new directions beyond what LoRA can cover.
In all, \algoname offers an efficient alternative for post-training LLMs to reason.

\paragraph{Limitations and future work}
While we focus on reasoning in this work, future work could also explore \algoname as a post-training method to acquire skills other than reasoning. Given the plug-able nature of prefix caches, \algoname opens the door for the composition of multiple skills with benefits including 1) avoiding catastrophic forgetting given prefix vectors leave existing weights intact, 2) inherently interpretable architecture via each skill's prefix attention scores, 3) adaptable prefix sizes for more or less complex skills. 
\algoname also offers possibilities for designing new flexible forms of memory. 
\algoname provides \textit{a spectrum of memory between short-term transient in-context and long-term persistent in-weight memory}. Additionally, the depth at which \algoname is applied can control the flow of abstraction from finer-grained information closer to the raw inputs at earlier layers to more abstract information distilled through deeper layers. As \algoname preserves compatibility with both ICL and in-weight methods, future work could also explore how \algoname can be best combined with existing methods to further enhance steering and skill acquisition in LLMs.
Finally, current inference engines and post-training frameworks do not natively support trainable KV-caches. For broad practical adoption, future work could also consider integrating trainable KV-caches into existing inference engines and post-training libraries using mechanisms similar to KV-caching.\

\section{Acknowledgements}
This work was supported by a Packard Fellowship to P.I., by the MIT-IBM Watson AI Lab, and by ONR MURI grant N00014-22-1-2740. This work was also supported by the NSF AI Institute TILOS (NSF CCF2112665) and the Alexander von Humboldt Foundation. S.G. acknowledges the support of the MathWorks Engineering Fellowship. We thank Reyhane Askari, Divyat Mahajan, Sachin Goyal, Andrei Nicolicioiu and Sarthak Mittal for insightful discussions.

\newpage
{
  \bibliographystyle{plainnat}
  \bibliography{references}
}

\appendix
\crefname{appendix}{Appendix}{Appendices}
\Crefname{appendix}{Appendix}{Appendices}
\crefalias{section}{appendix}
\crefalias{section}{appendix}
\crefalias{subsection}{appendix}
\crefalias{subsubsection}{appendix}
\startcontents[appendices]
\printcontents[appendices]{}{1}{\section*{Appendix}}
\newpage

\section{Supplementary Proofs}
\label{sec:proofs}

This appendix provides complete proofs for all theoretical results in \Cref{sec:theory}.

\subsection{Proof of~\Cref{prop:lora} (LoRA Novelty Subspaces)}

\loraprop*
\begin{proof}
($\Rightarrow$) Assume $U = \Pi_X \,\mathrm{span}\{v_1', \ldots, v_n'\}$ for some $\Delta_V$ with $\rank(\Delta_V) \le r$, where $v_i' = x_i(W_V + \Delta_V)$. The dimension of a projected subspace cannot exceed the dimension of the original subspace:
\[
\dim U = \dim(\Pi_X\,\mathrm{row}(X\Delta_V)) \le \dim(\mathrm{row}(X\Delta_V)) = \rank(X\Delta_V).
\]
By properties of matrix rank, $\rank(X \Delta_V) \le \min\{\rank(X), \rank(\Delta_V)\}$. Given $\rank(X) = t_X$ and $\rank(\Delta_V) \le r$, we have $\dim U \le \min\{t_X, r\}$.

($\Leftarrow$) Let $U \subseteq S_X^\perp$ be a subspace with dimension $s := \dim U \le \min\{t_X, r\}$. We construct a matrix $\Delta_V$ with $\rank(\Delta_V) \le r$ such that $\Pi_X \,\mathrm{span}\{v_1', \ldots, v_n'\} = U$.

Let $\{b_1, \ldots, b_s\}$ be an orthonormal basis for $U$, and form the matrix $B \in \mathbb{R}^{d \times s}$ whose columns are these basis vectors, so $\mathrm{col}(B) = U$. Since $s \le t_X = \rank(X)$, there exists a matrix $C \in \mathbb{R}^{d \times s}$ such that $XC \in \mathbb{R}^{n \times s}$ has rank $s$. To see this constructively, let the SVD of $X$ be $U_X \Sigma_X V_X^\top$; choosing $C$ to be the first $s$ columns of $V_X$ guarantees that $XC$ has rank $s$.

Define $\Delta_V := C B^\top$. The rank satisfies $\rank(\Delta_V) \le \min\{\rank(C), \rank(B^\top)\} = s \le r$. The resulting novelty matrix is $X \Delta_V = (XC) B^\top$. Since $XC$ has full column rank $s$ and $B^\top$ has full row rank $s$, we have $\rank(X \Delta_V) = s$ and $\mathrm{span}\{(X\Delta_V)_1, \ldots, (X\Delta_V)_n\} = U$.

Finally, since $U \subseteq S_X^\perp$, the projector $\Pi_X$ acts as the identity on $U$:
\[
\Pi_X \,\mathrm{span}\{v_1', \ldots, v_n'\} = \Pi_X(U) = U. \qedhere
\]
\end{proof}

\subsection{Proof of~\Cref{prop:pt} (Prefix Tuning Novelty Subspaces)}

\ptprop*
\begin{proof}
($\Rightarrow$) Assume $U = \Pi_X \,\mathrm{span}\{(P_V)_1, \ldots, (P_V)_m\}$. Then $\dim U \le \dim(\mathrm{span}\{(P_V)_1, \ldots, (P_V)_m\}) = \rank(P_V)$. Since $P_V$ has $m$ rows, its rank is at most $m$, so $\dim U \le m$.

($\Leftarrow$) Let $U \subseteq S_X^\perp$ be a subspace with dimension $s := \dim U \le m$. Let $\{u_1, \ldots, u_s\}$ be a basis for $U$. Construct $P_V \in \mathbb{R}^{m \times d}$ by setting its first $s$ rows to these basis vectors and the remaining $m - s$ rows to zero. Then $\mathrm{span}\{(P_V)_1, \ldots, (P_V)_m\} = U$. Since $U \subseteq S_X^\perp$, the projector $\Pi_X$ acts as the identity on $U$:
\[
\Pi_X \,\mathrm{span}\{(P_V)_1, \ldots, (P_V)_m\} = \Pi_X(U) = U. \qedhere
\]
\end{proof}

\subsection{Proof of~\Cref{thm:main} (Comparison of Expressivity)}
\mainthm*
\begin{proof}
By \Cref{prop:lora} and \Cref{prop:pt}, the families $\mathcal{L}(r)$ and $\mathcal{P}(m)$ are precisely the sets of all subspaces of $S_X^\perp$ whose dimensions are bounded by $\min\{t_X, r\}$ and $m$, respectively:
\[
\mathcal{L}(r) = \{U \subseteq S_X^\perp : \dim U \le \min\{t_X, r\}\}, \quad
\mathcal{P}(m) = \{U \subseteq S_X^\perp : \dim U \le m\}.
\]
Since any realizable subspace $U$ must reside in the ambient novelty space $S_X^\perp$, its dimension is inherently capped by $\nu_X = \dim S_X^\perp$. Thus, the maximum achievable dimensions are $D_{\text{LoRA}} = \min\{t_X, r, \nu_X\}$ and $D_{\text{PT}} = \min\{m, \nu_X\}$.

The set of all subspaces of a vector space up to a certain dimension forms a nested hierarchy: a family defined by cap $D_1$ is a subset of another with cap $D_2$ if and only if $D_1 \le D_2$. This immediately implies:
\begin{align*}
\mathcal{L}(r) \subseteq \mathcal{P}(m) &\iff D_{\text{LoRA}} \le D_{\text{PT}}, \\
\mathcal{P}(m) \subseteq \mathcal{L}(r) &\iff D_{\text{PT}} \le D_{\text{LoRA}}.
\end{align*}
Strict inclusion $\mathcal{L}(r) \subset \mathcal{P}(m)$ holds if and only if $\mathcal{L}(r) \subseteq \mathcal{P}(m)$ and $\mathcal{L}(r) \neq \mathcal{P}(m)$, which is equivalent to $D_{\text{LoRA}} < D_{\text{PT}}$. In this case, there exists a subspace $U \subseteq S_X^\perp$ of dimension $D_{\text{PT}}$ belonging to $\mathcal{P}(m)$ but not to $\mathcal{L}(r)$.
\end{proof}

\subsection{Proof of~\Cref{thm:hierarchy} (Loss Hierarchy)}
\hierarchythm*
We formalize the setting of the illustrative example in \Cref{sec:learnability}. Consider a single attention layer with frozen projections $W_Q, W_K, W_V \in \mathbb{R}^{d \times d}$. Suppose the context $X \in \mathbb{R}^{n \times d}$ has rank one, so $X = a u^\top$ for some $a \in \mathbb{R}^n$ and unit vector $u \in \mathbb{R}^d$. Each token $i$ belongs to one of two classes, with label $c_i \in \{1, 2\}$ determined by $c_i = 1$ if $a_i > 0$ and $c_i = 2$ if $a_i < 0$. Let $n_j := |\{i : c_i = j\}|$ denote the number of tokens in class $j$.

The target output for token $i$ is $y_i^\star := e_{c_i}$, where $e_1, e_2 \in \mathbb{R}^d$ are orthogonal unit vectors in $S_X^\perp$. The target matrix $Y^\star \in \mathbb{R}^{n \times d}$ thus has rank 2. The loss is $\mathcal{L}(Y) := \frac{1}{2}\|Y - Y^\star\|_F^2$.

We prove each part of the theorem separately.

\subsubsection{Part (i): QK-LoRA Loss Floor}

\begin{lemma}
\label{lem:qk-lock}
Under any variation $(\delta W_Q, \delta W_K)$, the output variation satisfies $\delta Y = (\delta A) V$, so each row $\delta y_i \in S_X$.
\end{lemma}

\begin{proof}
The value matrix $V$ does not depend on $W_Q$ or $W_K$. Differentiating $Y = AV$ gives $\delta Y = (\delta A) V$; each row is a linear combination of rows of $V$, hence lies in $S_X = \mathrm{span}\{v_1, \ldots, v_n\}$.
\end{proof}

\begin{proof}[Proof of Part (i)]
Let $E := Y - Y^\star$ and decompose $E_i = E_{i,\parallel} + E_{i,\perp}$ with $E_{i,\parallel} \in S_X$ and $E_{i,\perp} \in S_X^\perp$. Let $\theta := (W_Q, W_K)$ and let $J := \frac{\partial \mathrm{vec}(Y)}{\partial \theta}$ be the Jacobian.

By \Cref{lem:qk-lock}, each row of any infinitesimal output change $\delta Y$ lies in $S_X$, so $\langle E_\perp, \delta Y \rangle_F = 0$ for all variations. Since $\mathrm{vec}(\delta Y) = J \delta\theta$, we have $\mathrm{vec}(E_\perp)^\top J \delta\theta = 0$ for all $\delta\theta$, implying $J^\top \mathrm{vec}(E_\perp) = 0$.

The gradient of the loss is:
\[
\nabla_\theta \mathcal{L} = J^\top \mathrm{vec}(E) = J^\top \mathrm{vec}(E_\parallel) + \underbrace{J^\top \mathrm{vec}(E_\perp)}_{= 0} = J^\top \mathrm{vec}(E_\parallel).
\]
Thus, the gradient has no component that can reduce the novelty error $E_\perp$. Since $y_i^\star \in S_X^\perp$ and $y_i \in S_X$, we have $E_{i,\perp} = -y_i^\star$, giving:
\[
\mathcal{L}(Y) \ge \frac{1}{2} \|E_\perp\|_F^2 = \frac{1}{2} \sum_{i=1}^n \|y_i^\star\|^2 > 0. \qedhere
\]
\end{proof}

\subsubsection{Part (ii): QKV-LoRA Loss Floor}

\begin{proof}
With $W_V' = W_V + \Delta_V$ trainable, we have $X W_V' = (au^\top)(W_V + \Delta_V) = a \cdot (u^\top W_V + u^\top \Delta_V)$. Since $u^\top W_V + u^\top \Delta_V$ is a $1 \times d$ row vector, all rows of $X W_V'$ are scalar multiples of this vector. Thus $\rank(X W_V') \le 1$, and consequently $\rank(Y) \le 1$ for all parameter settings.

The target $Y^\star$ has $\rank(Y^\star)=2$, and its two nonzero singular values are
$\sqrt{n_1}$ and $\sqrt{n_2}$ (equivalently, in an orthonormal basis containing $e_1,e_2$,
$Y^\star$ has exactly two nonzero columns, which are orthogonal indicator vectors of norms
$\sqrt{n_1}$ and $\sqrt{n_2}$). By Eckart--Young--Mirsky, the best rank-1 approximation error in Frobenius norm is:
\[
\min_{\rank(Y) \le 1} \frac{1}{2} \|Y - Y^\star\|_F^2 = \frac{1}{2} \sigma_2(Y^\star)^2 = \frac{1}{2} \min\{n_1, n_2\} > 0. \qedhere
\]
\end{proof}

\subsubsection{Part (iii): Prefix Tuning Drive the Loss to Zero}

\begin{proof}
The proof is constructive. First, we show the queries are linearly separable. The query vectors are $q_i = a_i (u^\top W_Q)$ where each $q_i$ is a scalar multiple of the fixed direction $u^\top W_Q$. Define $\beta := (u^\top W_Q)^\top \in \mathbb{R}^d$. Assuming $\beta \neq 0$ (non-degeneracy), we have $\langle q_i, \beta \rangle = a_i \|\beta\|^2$, so $\mathrm{sign}(\langle q_i, \beta \rangle) = \mathrm{sign}(a_i)$. By the class definition, $c_i = 1$ iff $a_i > 0$ iff $\langle q_i, \beta \rangle > 0$, with margin $\gamma := \|\beta\|^2 \min_i |a_i| > 0$.

We construct a solution with prefix values $(P_V)_1 = e_1$, $(P_V)_2 = e_2$ and prefix keys $(P_K)_1 = \tau \beta$, $(P_K)_2 = -\tau \beta$ for large $\tau > 0$. For a class-1 token ($\langle q_i, \beta \rangle \ge \gamma$), the prefix logits satisfy $s'_{i1} = \frac{\tau}{\sqrt{d}} \langle q_i, \beta \rangle \to \infty$ and $s'_{i2} \to -\infty$ as $\tau \to \infty$. The attention weight on the first prefix converges to:
\[
A'_{i,n+1} = \frac{e^{s'_{i1}}}{\sum_j e^{s'_{ij}}} \to 1 \quad \text{as } \tau \to \infty.
\]
Symmetrically, class-2 tokens have $A'_{i,n+2} \to 1$. Thus, $y_i \to (P_V)_{c_i} = e_{c_i} = y_i^\star$ for all $i$, and $\mathcal{L}(Y) \to 0$.
\end{proof}

\section{Supplementary Experimental Details and Assets Disclosure}\label{sec:exp_setup}

\subsection{Assets}
We do not introduce new data in the course of this work. Instead, we use publicly available, widely used image datasets for the purposes of benchmarking and comparison.

\subsection{Hardware and Setup}
\label{sec.compute}
Each model was trained with distributed training on 8 NVIDIA H200 GPUs (144 GB memory each) paired with Intel Xeon Platinum 8488C CPUs (48 cores). Training was implemented in PyTorch using the TRL library~\citep{trl}, and inference was run in parallel on 4 H200 GPUs. For evaluation, we used the Language Model Evaluation Harness (\texttt{lm-evaluation-harness}) library~\citep{eval-harness}.

\subsection{Experimental Setup and Training Protocol}\label{add:exp_setup}
\subsubsection{Training Datasets}\label{sec:training}
We rely on two large-scale training corpora that expose models to rich chains of mathematical reasoning.  

MetaMathQA~\citep{yu2023metamath} is a dataset of roughly 400k math problems, each paired with detailed step-by-step solutions covering algebra, geometry, number theory, and calculus. We use it to provide broad coverage of high-school and undergraduate mathematics.  

OpenThoughts-3~\citep{guha2025openthoughts} contains approximately 3 million long-form reasoning traces across mathematics, physics, and logical puzzles. To make training feasible while retaining extended reasoning structure, we take a filtered subset of this dataset by discarding examples whose reasoning traces exceed 4096 tokens (as discussed in~\citep{guha2025openthoughts}). This reduces overall size while preserving the majority of long, structured derivations.

\subsubsection{Evaluation Benchmarks}\label{sec:evaluation}

Our evaluation spans two categories of reasoning tasks that capture distinct levels of difficulty: \emph{high-school mathematics}, which tests structured multi-step problem solving, and \emph{advanced competition or graduate-level problems}, which demand deeper abstraction and domain expertise.  

\paragraph{High-school mathematics.}  
GSM8K~\citep{cobbe2021training} consists of 8.5k grade-school word problems authored by human tutors. Each requires multi-step arithmetic reasoning and is paired with gold chain-of-thought solutions.  
MATH~\citep{hendrycks2020measuring} contains 12k competition-style problems spanning algebra, geometry, number theory, combinatorics, precalculus, and calculus. We follow prior work in evaluating on the standardized subset used in OpenAI’s benchmarks.  

\paragraph{Advanced reasoning.}
GPQA-Diamond~\citep{rein2024gpqa} contains 198 graduate-level questions in mathematics, physics, chemistry, and computer science. These items are sourced from qualifying exams and research-level material. Expert performance remains far from perfect (PhD holders achieve only 69.7\%), making it a stringent test of advanced reasoning~\citep{reasoningopenai}. When we write "GPQA" in the context of evaluation in this work, we always refer to the Diamond subset.

\subsubsection{Training and Evaluation Protocol}
\textbf{Training Details.} All models are trained with AdamW and a cosine learning-rate schedule with warmup ratio 0.05. Weight decay is set to zero throughout. On OpenThoughts-3, we train for 13 epochs with batch size 32 and a maximum sequence length of 8192. The learning rate is $1\text{e}{-5}$ for supervised fine-tuning, and $1\text{e}{-3}$ for prefix tuning, LoRA, and prompt tuning. On MetaMathQA, we train for 3 epochs with batch size 128 and a maximum sequence length of 2048, using a learning rate of $2\text{e}{-5}$ for all methods.  

For LoRA, we sweep over ranks $\{8,16,32,64,128\}$ and over different insertion points, including $W_Q,W_K$, full $W_{QKV}$, and $W_{QKV}$ plus feedforward layers, and report the best-performing configuration. For prompt tuning and prefix tuning, we sweep over the number of learned tokens in $\{16,32,64,128,256,512\}$ and report the best configuration. We additionally study two initialization techniques. First, \emph{reparameterization}, where prefixes are generated from a smaller embedding projected through an MLP, improving training stability. Second, \emph{initialization strategies}: prefixes are either initialized from a random token embedding in the base model’s vocabulary, or from a random example in the training dataset.  

\textbf{Evaluation Details.} At inference, all experiments are evaluated using the \texttt{lm-evaluation-harness} framework~\citep{eval-harness}, with temperature set to $0$ (greedy decoding). Accuracy, equivalent to pass@1, is reported as the primary metric. For GSM8K and MATH, we rely on the standard evaluation configurations provided in the harness.  

For models trained on the filtered OpenThoughts-3 dataset and evaluated on GPQA-Diamond, we adopt \emph{budget forcing} at inference~\citep{muennighoff2025s1}. Prefix tuning is not yet supported in \texttt{vLLM}, and HuggingFace backends can become bottlenecked by very long generations. We believe this limitation has constrained large-scale testing and broader adoption of prefix tuning, confining most prior work to datasets with shorter or no reasoning traces.  To address this, we follow the strategy proposed in~\citet{muennighoff2025s1}: a decoding-time intervention that constrains the number of "thinking" tokens. Specifically, generation is terminated early by appending an end-of-thinking delimiter and "Final Answer:" once the budget is reached. In our experiments, we enforce a maximum reasoning budget of 4096 tokens for evaluations on GPQA.

\section{Additional Experiments}
\subsection{Ablating Design Choices for \algofullname}\label{sec:ablations}
We study two knobs: (1) how the prefix is parameterized; and (2) how it is initialized; and measure their effect under varying parameter budgets. In both ablations, we finetune using \algoname on the complete MetaMathQA dataset and optimize to convergence, sweeping parameter budgets to quantify how each design choice impacts performance.

\begin{figure}[!htb]
    \centering
    \begin{minipage}{0.35\textwidth}
    \centering
    \includegraphics[width=\textwidth]{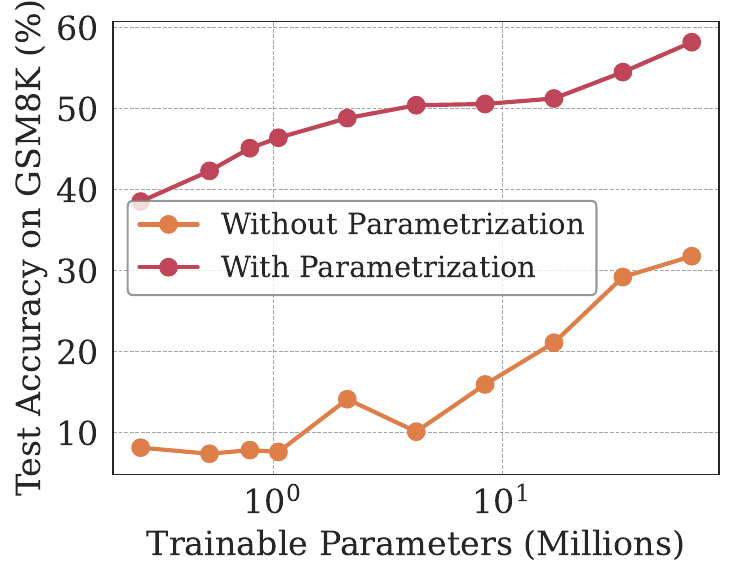}
    \caption{Test accuracy as a function of trainable parameters for \algofullname with direct optimization of prefix vectors and with MLP parametrization. Parametrization consistently improves accuracy, especially in the low-parameter regime.}
    \label{fig:acc_vs_parametrization}
    \end{minipage}
    \hspace{0.04\textwidth}
    \begin{minipage}{0.35\textwidth}
    \centering
    \includegraphics[width=\textwidth]{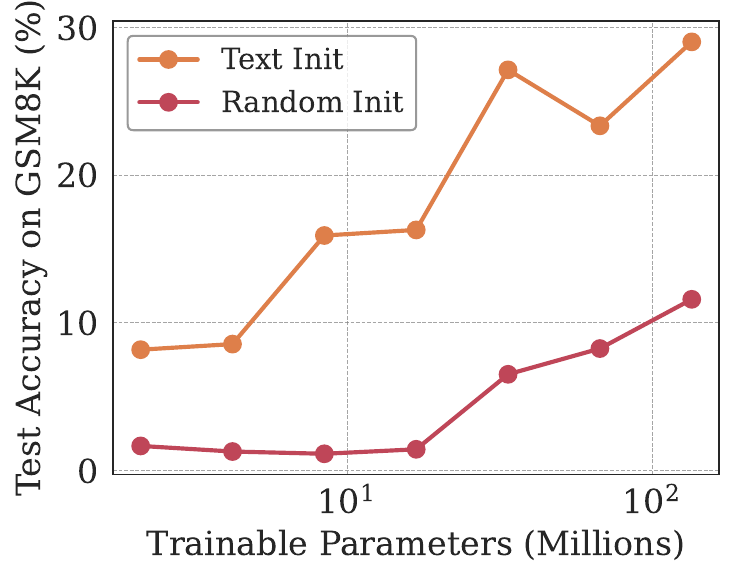}
    \caption{Test accuracy as a function of trainable parameters for \algofullname with random and text-based initialization. Initializing prefixes from text embeddings yields higher accuracy across all parameter budgets.}
    \label{fig:acc_vs_initialization}
    \end{minipage}
\end{figure}

\subsubsection{Effect of Prefix Parametrization}\label{sec: prefix parametrization}
\algofullname learns a set of virtual tokens (key-value vectors) at each layer. One can optimize these prefixes directly as free parameters, or generate them with a small network (e.g., an MLP) trained end-to-end via backpropagation. As shown in~\Cref{fig:acc_vs_parametrization}, MLP parametrization consistently improves test performance across all trainable-parameter budgets, with the largest gains in the low-budget regime. As we increase the parameter budget by increasing the prefix length, both variants improve, but parametrized prefixes retain a clear advantage throughout.

\subsubsection{Effect of Prefix Initialization}\label{sec: prefix initialization}
Initialization also plays a central role in \algofullname. Prefixes can be learned from scratch with random initialization, or initialized from text embeddings of few-shot exemplars. As shown in~\Cref{fig:acc_vs_initialization}, text-based initialization consistently improves accuracy across all trainable-parameter budgets. As we increase the prefix length, optimization improves both variants, but the advantage of text initialization persists.

\subsection{Probing The Representation Geometry for Expressivity}\label{sec:validate_theory}

\begin{figure}[!htb]
    \centering
    \begin{minipage}{0.46\textwidth}
    \centering
    \includegraphics[width=\textwidth]{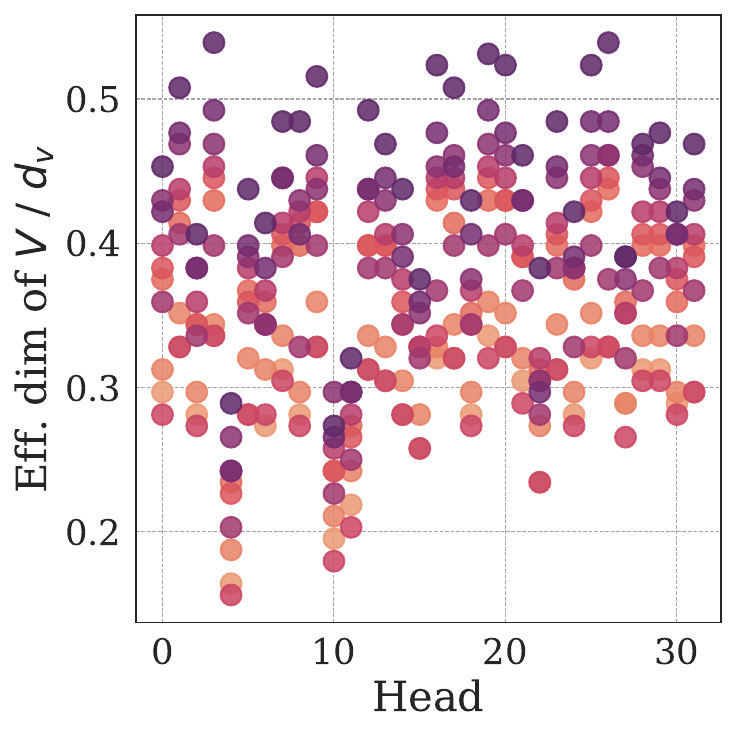}
    \caption{Effective dimension of value vectors across attention heads. Most heads use less than half of their available dimension, showing that the base model operates in a compact subspace.}
    \label{fig:app_effective_dimension}
    \end{minipage}
    \hfill
    \begin{minipage}{0.47\textwidth}
    \centering
    \includegraphics[width=\textwidth]{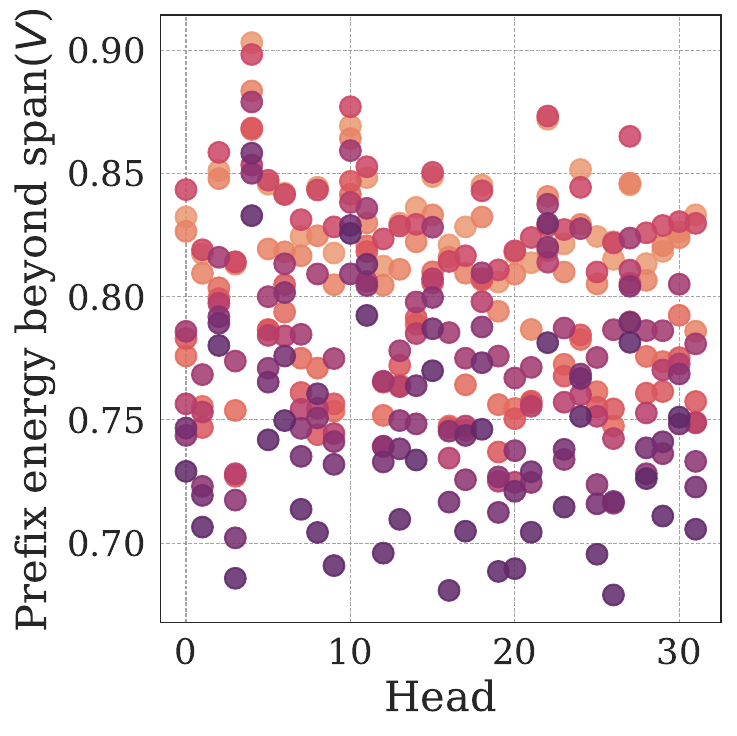}
    \caption{Fraction of prefix energy outside the dominant value subspace of the base model. Prefixes place 80\% of their mass beyond this span, exploiting directions unused by the base model.}
    \label{fig:app_prefix_energy}
    \end{minipage}
\end{figure}

Our theoretical analysis shows that when \algoname operates in regimes where in-weight updates are carrier-limited, it can realize its advantage by injecting value-space directions that lie outside the subspace used by the base model. We empirically validate this by probing representation geometry at three levels: (i) the dimensionality of value vectors used by the pretrained model, (ii) the directions of learned prefix values relative to this subspace, and (iii) the resulting dimensionality of downstream representations. All experiments in this section use LLaMA-2 7B as the base model, evaluated on GSM8K after adaptation on MetaMathQA.

\subsubsection{Base Models Operate In a Low-Dimensional Value Subspace.}
For each attention head in the final layer, we analyze the value matrix $V \in \mathbb{R}^{n \times d_v}$ induced by the pretrained model and compute its singular-value spectrum. Since $\sigma_i^2$ are the eigenvalues of $V^\top V$, they measure the variance explained along each direction. We define the \emph{effective dimension} of $V$ as the smallest $K$ such that
\[
\sum_{i=1}^{K} \sigma_i^2 \;\ge\; 0.9 \sum_{j=1}^{d_v} \sigma_j^2,
\]
reported as $K/d_v$.
Across heads and inputs, we find that the effective dimension of $V$, defined as the smallest $K$ capturing 90\% of the total variance, is consistently well below $d_v$ (\Cref{fig:app_effective_dimension}). This indicates that the base model concentrates computation in a compact value subspace, leaving substantial unused degrees of freedom for adaptation algorithms to exploit.

\subsubsection{Prefix Values Inject Energy Outside The Base Subspace}
We next examine whether \algoname exploits this geometric headroom. Let $V_P \in \mathbb{R}^{m \times d_v}$ denote the learned prefix values. Using the $K$-dimensional subspace of $V$ defined above, we project $V_P$ onto $\mathrm{span}(V)$ and measure the normalized residual energy
\[
\frac{\lVert V_P - \Pi_{\mathrm{span}(V)} V_P \rVert_F}{\lVert V_P \rVert_F}.
\]
This quantity captures the fraction of prefix energy that lies outside the value subspace actively used by the base model. As shown in~\Cref{fig:app_prefix_energy}, learned prefixes allocate the majority of their energy (typically 70--90\%) outside $\mathrm{span}(V)$, consistently across attention heads and input samples. Thus, \algoname does not merely reweight existing value directions but injects novel directions that are largely unused by the pretrained computation.

\subsubsection{Injected Novelty Propagates To Downstream Representations}
\begin{figure}[!htb]
    \centering
    \begin{minipage}{0.43\textwidth}
    \centering
    \includegraphics[width=\textwidth]{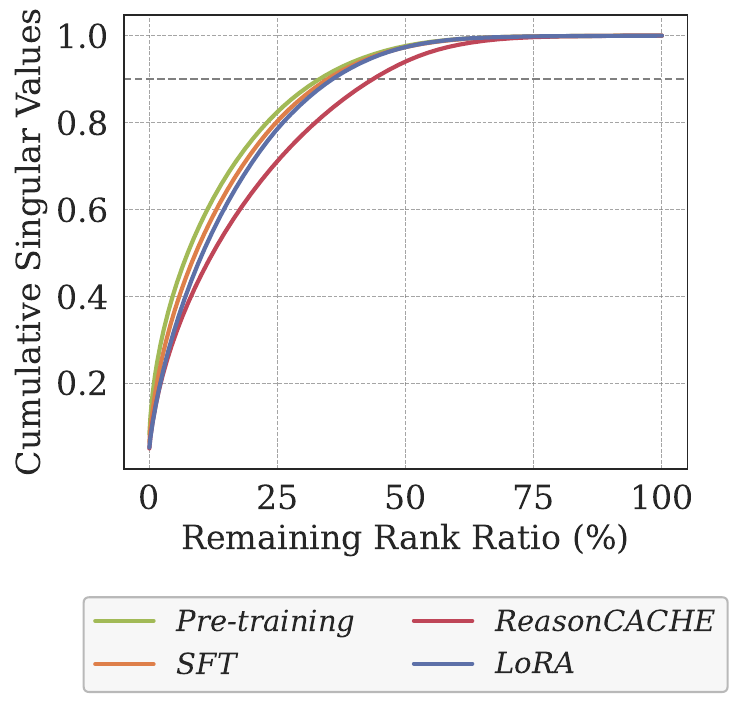}
    \caption{Cumulative singular-value spectra of last-layer token representations for GSM8K. \algoname yields a more gradual spectrum than the base model, SFT, and LoRA, indicating broader support across singular directions.}
    \label{fig:app_representation_collapse}
    \end{minipage}
    \hfill
    \begin{minipage}{0.47\textwidth}
    \centering
    \includegraphics[width=\textwidth]{figures/results/representation_collapse_ranks.pdf}
    \caption{Effective rank of last-layer token representations, defined as the rank at which the cumulative singular-value mass reaches $90\%$. \algoname attains the highest effective rank among all methods.}
    \label{fig:app_effective_rank}
    \end{minipage}
\end{figure}

We finally ask whether these newly injected directions persists beyond value space and reshapes the geometry of the model’s downstream representations. Concretely, we compute the singular-value spectrum of the last-layer token representation matrix $F$ and normalize it to form a distribution over singular directions. Averaging this normalized spectrum across the evaluation set and taking its cumulative mass yields the curves in~\Cref{fig:app_representation_collapse}. We summarize these curves via an \emph{effective rank}, defined as the smallest $k$ whose averaged cumulative mass reaches $90\%$.

As shown in~\Cref{fig:app_effective_rank}, \algoname produces a markedly more gradual spectrum and a higher effective rank than both supervised fine-tuning and LoRA. This indicates that the directions injected at the attention level do not vanish downstream; instead, they propagate through the network and are expressed in the final representations as broader support across singular directions. 

Together, these results corroborate our theoretical analysis. The pretrained model indeed operates within a low-dimensional value subspace; \algoname injects directions orthogonal to this span; and these directions propagate to yield higher-dimensional downstream representations. 

\begin{figure}[!htb]
    \centering
    \begin{minipage}{.8\textwidth}
        \centering
        \begin{minipage}{0.32\textwidth}
            \centering
            \includegraphics[width=\textwidth]{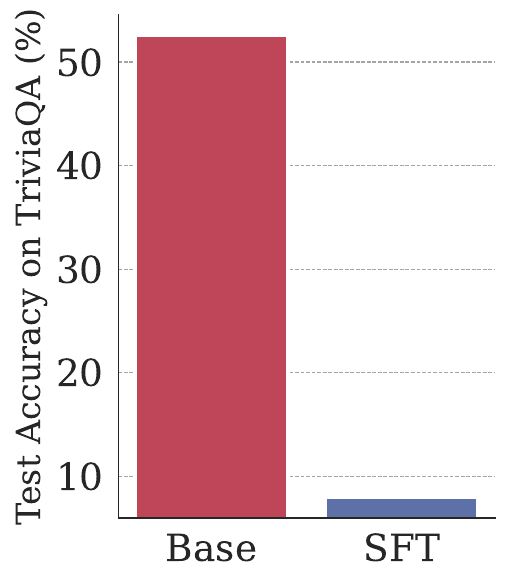}
        \end{minipage}
        \hfill
        \begin{minipage}{0.32\textwidth}
            \centering
            \includegraphics[width=\textwidth]{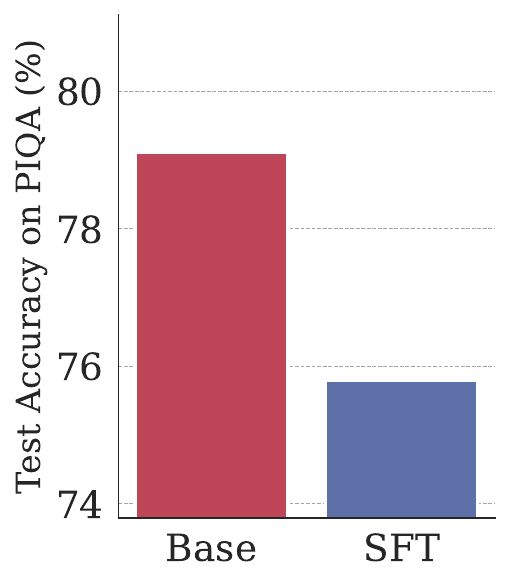}
        \end{minipage}
        \hfill
        \begin{minipage}{0.32\textwidth}
            \centering
            \includegraphics[width=\textwidth]{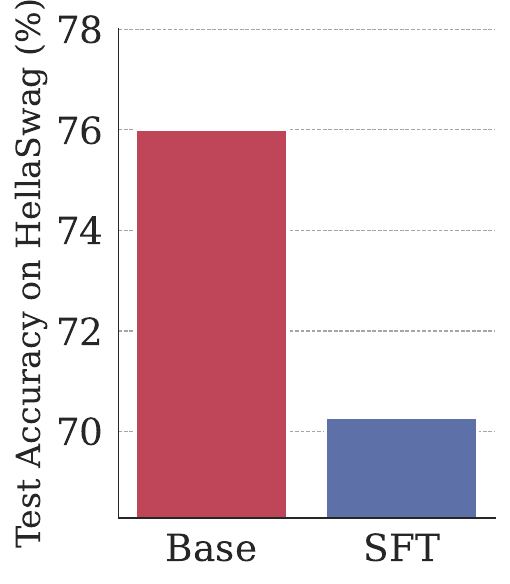}
        \end{minipage}
        \caption{Accuracy on TriviaQA, PIQA, and HellaSwag for LLaMA-2 7B before and after supervised fine-tuning (SFT) on MetaMathQA, illustrating catastrophic forgetting under in-weight adaptation.}
        \label{fig:forgetting}
    \end{minipage}
\end{figure}

\subsection{In-Weight Adaptation Exhibits Catastrophic Forgetting Unlike \algofullname}\label{sec: forgetting}

Finally, we test whether in-weight adaptation degrades general-purpose capabilities outside the training distribution.~\Cref{fig:forgetting} reports accuracy on TriviaQA, PIQA, and HellaSwag for LLaMA-23 7B before and after supervised fine-tuning (SFT) on MetaMathQA. SFT substantially reduces performance on all three benchmarks relative to the base model, consistent with catastrophic forgetting from task-specific weight updates~\citep{kirkpatrick2017overcoming}.

By contrast, \algoname and LoRA leave the pretrained weights unchanged and are reversible by construction: removing the learned prefixes or adapters exactly recovers the base model. We therefore omit them from \Cref{fig:forgetting}, since their post-removal performance matches the base model. Thus, while SFT can improve task performance but may overwrite broad capabilities, whereas parameter-isolated adaptation preserves the base model while enabling similar to better task-specific performance as shown in~\Cref{fig:results_main_fig}.

\section{Extended Related Work}
\label{sec:related_work_appendix}

This section provides an extended discussion of related work, expanding on the summary in \Cref{sec:related_work}.

\paragraph{In-Context Learning and Its Limitations.}
In-context learning (ICL) enables LLMs to adapt to new tasks from demonstrations without gradient updates~\citep{brown2020language}. While remarkably sample-efficient for format control and simple tasks~\citep{min2022rethinking}, ICL faces fundamental limitations when scaling to tasks that require more demonstrations such as complex reasoning. These limitations manifest along several axes: \emph{positional bias}, where models struggle to access information in the middle of long contexts~\citep{liu2024lost}; \emph{stability degradation}, where increasing context length introduces attention noise and unreliable performance~\citep{hong2025contextrot}; \emph{shallow learning}, where learning is primarily driven by deducing patterns from prompt regularities, resulting in limited generalization~\citep{de2025context}; and \emph{diminishing returns}, where as the number of ICL demonstrations increases from a few to many, the performance of LLMs initially improves but then plateaus and can even decline~\citep{agarwal2024many,zhang2025more,zhang2025memory}. These limitations motivate learning to compress demonstrations into a compact KV-cache that sidesteps the pathologies of long contexts.

\paragraph{Prompting and Eliciting Reasoning.}
Chain-of-thought prompting~\citep{wei2022chain} showed that intermediate reasoning steps can elicit latent reasoning abilities. Extensions include least-to-most prompting~\citep{zhou2022least}, which decomposes problems into subproblems, self-consistency~\citep{wang2022self}, which improves reliability via multiple sampled chains, and zero-shot chain-of-thought~\citep{kojima2022large}, which activates reasoning with simple prompt phrases. More recently, schema-activated ICL~\citep{chen2025schema} has proposed retrieving abstract reasoning templates to bridge the gap between pattern priming (surface-level mimicry) and complex reasoning.

\paragraph{Prompt Tuning and Prefix Tuning.}
Prompt tuning~\citep{lester2021power} introduced learnable continuous embeddings prepended at the input layer, enabling task adaptation without modifying pretrained weights. Prefix tuning~\citep{li2021prefix} extended this by learning key--value vectors at every attention layer, providing greater representational flexibility. P-tuning v2~\citep{liu2021p} showed that deep prompt tuning can match full fine-tuning across scales and tasks. Recent work has analyzed the statistical benefits of prefix reparameterization~\citep{le2024revisiting}, proposed architectural improvements such as decoupling the attention paid to the prefixes and regular context~\citep{wang2025prefix}, and used Bayesian frameworks to demonstrate that soft prefixes can manipulate activations in ways inaccessible to discrete tokens~\citep{genewein2025understanding}. While these methods were primarily evaluated on classification and generation tasks, we study whether prefix tuning can acquire complex reasoning skills---a capability typically associated with weight updates.

\paragraph{Context Compression and Distillation.}
Several methods aim to compress long contexts into compact representations. Gist tokens~\citep{mu2023learning} train an LM to compress prompts into a small number of virtual tokens that can be cached and reused. AutoCompressor~\citep{chevalier2023adapting} adapts LMs to compress long contexts into compact summary vectors that are used as soft prompts. LLMLingua~\citep{jiang2023llmlingua} performs coarse-to-fine prompt compression by selecting/pruning tokens using LM-scoring signals (e.g., perplexity), while LLMLingua-2~\citep{pan2024llmlingua} uses data distillation to learn task-agnostic compression policies. Natural language compression~\citep{chuang2024learning} produces shorter textual ``capsule prompts'' for better transferability across LLMs. ICAE~\citep{ge2023context} compresses long context into compact memory slots via an in-context autoencoding objective. Cartridges~\citep{eyuboglu2025cartridges} compress massive contexts into learnable KV-caches via self-study, allowing fast inference with frozen backbones. Deep context distillation~\citep{caccia2025training} trains document-level knowledge modules (LoRA modules) to simulate a teacher model's hidden states and logits when the teacher has full document access. These methods primarily compress or modularize \emph{contextual information}, rather than learning a persistent complex skill (e.g., reasoning) beyond what is expressed in the compressed representation.

\paragraph{Memory and Continual Learning.}
Our work relates to the broader question of how LLMs can acquire and retain knowledge. Standard in-weight learning comes with its own issues, such as the reversal curse~\citep{berglund2023reversal}, where models fail to generalize bidirectional relationships. In contrast, Lampinen et al.~\citep{lampinen2025generalization} demonstrate that in-context learning naturally avoids these failures, generalizing flexibly where in-weight fine-tuning breaks down. However, ICL is ephemeral and limited by context length.
Continual learning via sparse memory fine-tuning~\citep{lin2025continual} updates only the most relevant memory slots to prevent catastrophic forgetting. Memory Mosaics~\citep{zhang2024memory} replace standard attention with networks of associative memories, demonstrating superior compositional generalization. Our perspective differs: we view the KV-cache as a \emph{learnable memory interface} that spans a spectrum from raw tokens (recovering ICL) to gradient-optimized vectors (approaching in-weight learning). This framing opens a broader design space for adaptation, where learned prefixes serve as modular, pluggable ``skill caches'' that can be composed or swapped without modifying the base model. Prefix tuning, however, is \emph{not} a true continual/test-time learner: the KV-cache is trained offline and then frozen. Developing methods where the KV-cache remains plastic at test time is an exciting direction for future work.

\paragraph{In-Context Optimization and Test-Time Learning.}
Recent work has framed transformers as performing implicit optimization over context~\citep{von2023transformers,mittal2025iterative}. Context tuning~\citep{lu2025context} initializes learnable prefixes from ICL demonstrations, showing that optimization works best when grounded in context. Function vectors~\citep{todd2023function} identify that specific attention heads transport compact task representations, providing mechanistic insight into why prefix tuning works. Generative adapters~\citep{chen2024generative} use hypernetworks to generate LoRA adapters based on input context.

\paragraph{Reasoning via Fine-Tuning and Reinforcement Learning.}
The dominant approach for instilling reasoning capabilities involves weight updates. DeepSeek-R1~\citep{guo2025deepseek} and OpenAI's reasoning models~\citep{reasoningopenai} use reinforcement learning to incentivize reasoning. Data curation efforts~\citep{guha2025openthoughts,yu2023metamath,ye2025limo} focus on constructing high-quality reasoning traces. Methods such as s1~\citep{muennighoff2025s1} and LIMO~\citep{ye2025limo} show that a small number of high-quality examples ($<$1000) can instill reasoning skills, though individual examples can be extremely long ($>$10k tokens). Although prior theoretical work suggests that prefix tuning may be less expressive than full fine-tuning~\citep{petrov2023prompting}, our work shows that prefix tuning offers an alternative post-training interface: learning reasoning skills into a portable KV-cache that can be plugged into a frozen model, combining the modularity of in-context methods with the depth of weight-based learning.


\end{document}